\documentclass{article}

\usepackage[main,final]{neurips_2026}

\usepackage[utf8]{inputenc} 
\usepackage[T1]{fontenc}    
\usepackage[pdfencoding=auto, psdextra, pagebackref=true]{hyperref}

\usepackage{url}            
\usepackage{booktabs}       
\usepackage{amsfonts}       
\usepackage{nicefrac}       
\usepackage{microtype}      
\usepackage{xcolor}         

\newcommand{\declarecolor}[2]{\definecolor{#1}{RGB}{#2}\expandafter\newcommand\csname #1\endcsname[1]{\textcolor{#1}{##1}}}
\declarecolor{White}{255, 255, 255}
\declarecolor{Black}{0, 0, 0}
\declarecolor{Maroon}{128, 0, 0}
\declarecolor{Coral}{255, 127, 80}
\declarecolor{Red}{182, 21, 21}
\declarecolor{LimeGreen}{75, 205, 75}
\declarecolor{DarkGreen}{0, 80, 0}
\declarecolor{Purple}{146, 42, 158}
\declarecolor{Navy}{0, 0, 128}
\declarecolor{LightBlue}{84, 101, 202}
\definecolor{mydarkblue}{rgb}{0,0.08,0.45}
\usepackage{tcolorbox}

\usepackage{makecell}

\definecolor{CardinalRed}{HTML}{C41E3A}
\definecolor{Dartmouth}{HTML}{00693E}
\definecolor{SapphireBlue}{HTML}{0F52BA}

\colorlet{MyRed}{CardinalRed}
\colorlet{MyGreen}{LimeGreen}

\colorlet{MyLightRed}{MyRed!25}
\colorlet{MyLightGreen}{MyGreen!25}

\colorlet{AlertColor}{MyRed}	
\colorlet{BadColor}{MyRed}	
\colorlet{FocusColor}{MyRed}	
\colorlet{GoodColor}{MyGreen}	
\colorlet{MacroColor}{MyRed}	%

\hypersetup{
    colorlinks,
    linkcolor={blue!50!black},
    citecolor={blue!50!black},
    urlcolor={blue!80!black}
}

\usepackage{tikz}
\usetikzlibrary{arrows.meta,positioning,calc}

\makeatletter
\def\blfootnote{\gdef\@thefnmark{}\@footnotetext}
\makeatother

\usepackage{algorithm,algorithmic}
\usepackage{amsmath,amsfonts,amssymb}
\usepackage{mathtools}
\usepackage{amsthm} 
\usepackage{latexsym}
\usepackage{relsize}
\usepackage{enumitem}

\usepackage{wrapfig}
\usepackage{float}

\usepackage[capitalise]{cleveref}
\usepackage{xcolor}
\usepackage{dsfont}

\usepackage{tcolorbox}

\def\X{{\mathcal X}}

\def\Y{{\mathcal Y}}

\newcommand{\A}{\mathcal{A}}

\newcommand{\K}{\ensuremath{\mathcal K}}

\newcommand{\ignore}[1]{}

\theoremstyle{plain}
\newtheorem{theorem}{Theorem}
\newtheorem{lemma}{Lemma}
\newtheorem{corollary}[theorem]{Corollary}
\newtheorem{proposition}{Proposition}

\newtheorem{assumption}{Assumption}

\newtheorem*{theorem*}{Theorem}
\newtheorem*{lemma*}{Lemma}
\newtheorem*{corollary*}{Corollary}
\newtheorem*{proposition*}{Proposition}
\newtheorem*{claim*}{Claim}
\newtheorem*{fact*}{Fact}
\newtheorem*{observation*}{Observation}
\newtheorem*{assumption*}{Assumption}

\theoremstyle{definition}
\newtheorem{remark}{Remark}
\newtheorem{example}{Example}
\newtheorem*{definition*}{Definition}
\newtheorem*{remark*}{Remark}
\newtheorem*{example*}{Example}

 \theoremstyle{plain}
\newtheorem*{theoremaux}{\theoremauxref}
\gdef\theoremauxref{1}

\DeclareMathAlphabet{\mathbfsf}{\encodingdefault}{\sfdefault}{bx}{n}

\DeclareMathOperator*{\argmin}{arg\,min}

\newcommand{\reals}{\mathbb{R}}

\newcommand{\eqdef}{:=}

\renewcommand{\leq}{~\le~}
\renewcommand{\geq}{~\ge~}

\let\oldtfrac\tfrac
\renewcommand{\tfrac}[2]{\smash{\oldtfrac{#1}{#2}}}

\let\nablaold\nabla
\renewcommand{\nabla}{\nablaold\mkern-2.5mu}

\title{Online Learning on Hidden-Convex Losses\\ via Algorithmic Equivalence: Optimal Regret, Geometric Barrier, and Bandit Feedback}

\author{
  \begin{tabular}{c}
    Anas Barakat$^{1}$
  \end{tabular}
  \hfil
  \hspace*{1em}
  \begin{tabular}{c}
    Andreas Kontogiannis$^{2,5}$
  \end{tabular}
  \hfil
  \hspace*{1em}
  \begin{tabular}{c}
    Vasilis Pollatos$^{3,5}$
  \end{tabular}
  \vspace{0.2em} \\ 
  \begin{tabular}{c}
    \textbf{Ioannis Panageas}$^{4}$
  \end{tabular}
  \hspace*{1em}
  \begin{tabular}{c}
    \textbf{Antonios Varvitsiotis}$^{1, 5, 6}$
  \end{tabular}
}

\begin{document}
\maketitle
\blfootnote{
    Affiliations:\;  
    $^{1}$ Singapore University of Technology and Design \;
    $^{2}$ National Technical University of Athens \;
    $^{3}$ National and Kapodistrian University of Athens \;\;
    $^{4}$ University of California, Irvine\;
    $^{5}$ Archimedes, Athena Research Center, Greece  \;\;
    $^{6}$ National University of Singapore, Centre for Quantum Technologies \;\;
  \vspace*{0.1em}
}
\blfootnote{
Contact: 
\texttt{barakat9anas@gmail.com,
    andr.kontog@gmail.com, 
    vaspoll97@gmail.com, \quad\\
    \hspace*{8em}
    ipanagea@ics.uci.edu,antonios@sutd.edu.sg. 
}}

\begin{abstract}
We study adversarial online learning with hidden-convex losses, i.e., nonconvex losses that become convex after a nonlinear reparameterization. 
Ghai, Lu, and Hazan \cite{ghai-lu-hazan22} proved that, under geometric and smoothness assumptions, online gradient descent (OGD) on such nonconvex losses approximately simulates online mirror descent (OMD) on the underlying convex losses with a suitable regularizer, yielding $\mathcal{O}(T^{2/3})$ regret. 
They left open whether the optimal $\Theta(\sqrt{T})$ regret from online convex optimization can be recovered in this hidden-convex setting. We answer this question affirmatively. More specifically, via a sharper discrete-time algorithmic equivalence argument, we prove that OGD achieves $\mathcal{O}(\sqrt{T})$ regret under the same assumptions, matching the optimal worst-case rate for adversarial online convex optimization. 
We also address another open question of \citep{ghai-lu-hazan22} by clarifying the geometry required for this algorithmic equivalence. 
We replace the diagonal-Jacobian sufficient condition with a necessary-and-sufficient Hessian compatibility condition, thereby expanding the class of admissible reparameterizations. 
We complement our tight regret bound with a lower bound showing that OGD suffers $\Omega(T)$ regret even
with a stepsize of order $T^{-1/2}$, for a smooth Hessian incompatible
reparameterization between convex domains and an adversarial sequence
of hidden-convex losses. 
Finally, we extend our analysis to one-point bandit feedback and prove a $\mathcal{O}(T^{3/4})$ expected regret bound for bandit OGD with spherical smoothing, matching its classical rate on convex losses.
\end{abstract}

\section{Introduction}

Online convex optimization (OCO) provides a theoretical framework for adversarial sequential decision-making: for convex Lipschitz losses, simple first-order methods such as online gradient descent (OGD) achieve the optimal $\Theta(\sqrt{T})$ regret rate \cite{zinkevich03,hazan16oco,shalev-shwartz12oco-ftml}.
Beyond convexity, the situation is substantially more delicate. 
For arbitrary nonconvex losses, sublinear regret against the best fixed decision in hindsight is generally unattainable by simple first-order methods without additional structure or oracle access \citep{krichene-et-al15hedge-continuum,agarwal-gonen-hazan19,suggala-netrapalli20,heliou-et-al20}. For instance, it is known that no deterministic algorithm can achieve sublinear regret in online nonconvex learning in general \cite[Proposition~3]{suggala-netrapalli20}. 
This motivates the search for structured classes of nonconvex online problems that retain enough geometry to permit convex-like guarantees with simple first-order algorithms.

A prominent form of benign nonconvexity is \textit{hidden convexity} \citep{bental-teboulle96hidden-cvx,li-et-al05hidden-cvx-min,fatkhullin-he-hu25stoch-hidden-cv,levin-kileel-boumal25,gorissen-et-al26hidden-cvx}. 
This structural property appears in several modern optimization problems, including neural-network training~\citep{wang-et-al22iclr,patelsolving,ergen-pilanci25,zeger-pilanci26}, reinforcement learning~\citep{hazan-et-al19,zhang-et-al20variational,zahavy-et-al21,barakat-et-al23rlgu,barakat-et-al25rlgu,dorazio2025solving}, and nonconvex games~\citep{vlatakis2019poincare,mladenovic-et-al22hidden,sakos2023exploiting,kalogiannis2024learning,gemp-et-al25cmg,kalogiannis-et-al25zs-cmg,barakat-et-al26gumg}. In these settings, the objective is nonconvex in the algorithm’s native parameters but convex after an appropriate nonlinear change of variables. 
This structure has been exploited in deterministic and stochastic optimization to obtain global convergence guarantees for nonconvex problems \cite{chen-et-al25,fatkhullin-he-hu25stoch-hidden-cv,fatkhullin-et-al25constrained-hidden-cv,bhaskara-et-al25}. 
In the online setting, however, the learner faces a distinct challenge: the loss functions which are revealed sequentially over time may even be chosen adversarially and the goal becomes to achieve low regret against the best fixed comparator in hindsight rather than convergence to the minimizer of a fixed function.

Ghai, Lu, and Hazan \cite{ghai-lu-hazan22} initiated the study of such a structured nonconvex online learning setting through algorithmic equivalence. They showed that, under suitable geometric and smoothness assumptions on the reparameterization, OGD applied to the nonconvex losses approximately simulates OMD applied to the corresponding convex losses in the reparameterized space. 
This connection led to a $\mathcal{O}(T^{2/3})$ regret bound. Their result left open whether the slower-than-convex rate is inherent to the nonconvex online learning setting. Since the identity reparameterization recovers ordinary OCO, the best possible regret one could hope for in this class with exact gradient feedback is $\Theta(\sqrt{T})$. 
This raises the central question: 
\begin{quote}
\center
\textit{
When does OGD recover the optimal regret under exact gradient feedback for nonconvex hidden-convex losses, and does the same algorithmic equivalence principle extend to the bandit feedback setting?
}
\end{quote}

In this paper, we answer the exact gradient feedback question affirmatively and show that the framework extends to one-point bandit feedback. 
Specifically, we prove a $\mathcal{O}(\sqrt{T})$ regret for OGD, discuss the required compatibility conditions needed for the algorithmic equivalence, show a linear-regret lower bound without it, and sublinear expected regret guarantees in the bandit setting. We elaborate on our contributions in more detail in the next section. 


\subsection{Main contributions}

We study online nonconvex learning for losses that are convex under a nonlinear reparameterization which is \textit{unknown} to the learner.  
Our results address several open questions formulated by Ghai, Lu, and Hazan \cite{ghai-lu-hazan22}, refining and extending their algorithmic equivalence analysis between OGD and OMD. Our main contributions are as follows:

\begin{itemize}[leftmargin=*]
    \item \textbf{Optimal regret under exact gradient feedback.}
    In the exact gradient setting, we prove that OGD achieves $\mathcal{O}(\sqrt{T})$ regret for hidden-convex losses (Theorem~\ref{thm:main}) under the geometric compatibility and smoothness assumptions considered in \cite{ghai-lu-hazan22}. 
    This improves their $\mathcal{O}(T^{2/3})$ bound and matches the optimal regret for adversarial online \textit{convex} optimization. 
    Since classical OCO is recovered by the identity reparameterization, the rate is optimal for this problem class in general. Our analysis is based on a sharper algorithmic equivalence between OGD and online mirror descent (OMD), improving the perturbation analysis underlying the previous $\mathcal{O}(T^{2/3})$ guarantee. Specifically, our proof shows that OGD on the nonconvex losses tracks OMD in the convex parameterization with sufficiently small discretization error to preserve the classical $\sqrt{T}$ regret. 

    \item \textbf{Geometry of the reparameterization.} 
    Prior work \citep{ghai-lu-hazan22} imposed a diagonal-Jacobian condition to guarantee the Hessian compatibility needed for the OGD–OMD equivalence. We relax this structural assumption and give necessary and sufficient conditions for compatibility (Proposition~\ref{prop:matrix-field-eq-Hessian}), thereby enlarging the class of reparameterizations covered by the analysis.

    \item \textbf{Geometric barrier for OGD.}
    We show that the Hessian compatibility condition is not merely a proof artifact. 
    We prove that there exists a smooth reparameterization violating compatibility and an adversarial sequence of hidden-convex losses for which OGD incurs $\Omega(T)$ regret (Theorem~\ref{thm:linear-reg-lower-bound}).

    \item \textbf{Bandit feedback.}
    We extend our analysis to the one-point bandit setting, where the learner only observes a single loss value per round. 
    Against oblivious adversaries, we prove an expected regret bound of 
    $\mathcal{O}(T^{3/4})$ under standard additional Lipschitzness and boundedness assumptions (Theorem~\ref{thm:bandit}). 
    This matches the classical guarantee for OGD with one-point spherical smoothing on convex losses
    \citep{flaxman-kalai-mcmahan05}. 
\end{itemize}
Together, our results establish new regret guarantees for OGD for a structured class of \textit{online nonconvex} optimization problems, extending classical results from online \textit{convex} optimization. 

\subsection{Related Work}

\noindent\textbf{Online nonconvex optimization.}
Online convex optimization is a standard framework for adversarial sequential decision-making; see, e.g., \cite{hazan16oco,shalev-shwartz12oco-ftml,orabona19book-online-learning}. 
A separate line of work studies online learning with nonconvex losses \citep{krichene-et-al15hedge-continuum,agarwal-gonen-hazan19,suggala-netrapalli20,heliou-et-al20}. 
Sublinear global regret in such settings typically requires access to strong oracles, such as sampling oracles that allow continuous-domain exponential weights \citep{maillard-munos10,krichene-et-al15hedge-continuum}, or offline optimization oracles for the nonconvex losses \citep{agarwal-gonen-hazan19,suggala-netrapalli20}. 
Our work takes a different route: rather than assuming oracle access for general nonconvex losses, we exploit hidden convexity of the loss sequence and analyze first-order algorithms through the algorithmic equivalence between OGD in the nonconvex parameterization and OMD in the convex parameterization \citep{amid2020reparameterizing,ghai-lu-hazan22,li-et-al22implicit-bias}. A few works in the literature also consider structured nonconvex losses with special structure such as the composition of a non-increasing function with a linear function \citep{zhang-et-al15}, or weaker notions of convexity such as strictly locally quasi-convexity \citep{hazan-et-al15beyond-cvx} and weak pseudo-convexity \citep{gao-li-zhang18}. 

\noindent\textbf{Hidden-convex optimization.}
Hidden convexity has recently been studied in several offline deterministic and stochastic optimization settings \citep{chen-et-al25,fatkhullin-he-hu25stoch-hidden-cv,fatkhullin-et-al25constrained-hidden-cv,levin-kileel-boumal25}. 
These works consider nonconvex problems that admit a convex reformulation under a nonlinear transformation, often under assumptions weaker than those used in the online algorithmic-equivalence framework. 
Their focus, however, is offline optimization: the objective is fixed, or arises as an expectation over a fixed data-generating distribution. 
By contrast, we study an adversarial online setting in which the loss functions may change arbitrarily over time, and the performance criterion is regret rather than convergence to a global optimizer. 
This distinction requires a different analysis, since the learner must compete with the best fixed decision in hindsight while only observing feedback sequentially.

\noindent\textbf{Online learning under bandit feedback.}
The literature on online bandit learning is extensive; we refer to \cite{lattimore-szepesvari20,lattimore24banditCO} for modern treatments. 
For bandit online convex optimization, Flaxman, Kalai and McMahan \cite{flaxman-kalai-mcmahan05} introduced a one-point gradient estimator based on spherical smoothing, showing that gradient-based online methods can be implemented when only function values are observed. 
The bandit setting has also been studied for general Lipschitz nonconvex losses and for special nonconvex classes \citep{zhang-et-al15,gao-li-zhang18}. 
Our bandit results are different in focus: we consider hidden-convex losses and use smoothing together with the OGD--OMD equivalence to transfer regret guarantees from the convex parameterization to the original nonconvex one. 
Thus, while our algorithmic template is inspired by bandit OCO, the main challenge is controlling the interaction between bandit gradient estimation, nonlinear reparameterization, and the approximation error in algorithmic equivalence.

\section{Online Learning on Hidden-Convex Losses} 
\label{sec:OHCO}

In this section, we introduce the online \textit{hidden}-convex optimization (OHCO) problem which extends the
celebrated online convex optimization (OCO) setting \citep{shalev-shwartz12oco-ftml,hazan-et-al15beyond-cvx,orabona19book-online-learning}. 

\noindent\textbf{Hidden-convex losses.}  
Let $\X,\Y\subset\mathbb R^d$ be compact convex sets with
nonempty interiors.
Consider a sequence of loss functions $(\ell_t)$ where for any $t \geq 1,$ $\ell_t: \X \to \mathbb{R}$ 
has the following hidden-convexity structure: 
\begin{equation}
    \label{eq:hidden-convex}
\ell_t(x) = h_t(q(x))\,, \quad \forall x \in \X\,, 
\end{equation}
where $q: \X \to \Y$ is a smooth bijective reparameterization function and $h_t: \Y \to \mathbb{R}$ is a convex function. 
In this work, we will focus on a subclass of hidden-convex functions satisfying some geometric and smoothness assumptions made in \cite{ghai-lu-hazan22}. 
We defer the statement of these assumptions on~$q$ and~$(\ell_t)$ to Sections~\ref{sec:full-info}-\ref{sec:bandit-feedback}. 
While $\ell_t$ is nonconvex in general, it is convex under the  reparameterization~$q$ and we then say that the function $\ell_t$ is \textit{hidden-convex}.  
When the reparameterization $q$ is simply the identity, the function $\ell_t$ is a (truly) convex function. 

\textbf{Online regret minimization.} We are interested in the online optimization setting where a learner seeks to minimize an unknown sequence of hidden-convex functions $(\ell_t)$ by choosing a sequence
$(x_t)$ where each $x_t$ is based only on the information observed before round $t$. 
At each round $t$, a learner chooses $x_t\in \X$ using an online algorithm $\A$, then an adversary reveals a loss function $\ell_t$ which is hidden-convex (see \eqref{eq:hidden-convex}) 
and the player suffers the loss $\ell_t(x_t)$. 
The goal of the learner is to achieve a good performance with respect to the best possible fixed decision in hindsight,  minimizing the (external) regret of the algorithm $\A$ defined for any time horizon $T \geq 1$ as follows:
$$
R_T :=\sum_{t=1}^T \ell_t(x_t)-\min_{x\in \X} \sum_{t=1}^T \ell_t(x)\,.
$$
Since $q:\X\to\Y$ is bijective, the regret can equivalently be written in hidden coordinates as 
$R_T = \sum_{t=1}^T h_t(z_t) - \min_{z\in\Y}\sum_{t=1}^T h_t(z)$ where $z_t = q(x_t)\,.$ 
Unlike in the classical OCO setting, the sequence $(\ell_t)$ is not supposed to be convex. 
We recover the standard OCO setting by choosing $q$ to be the identity function in \eqref{eq:hidden-convex}. 

\noindent\textbf{Information settings.} 
The reparameterization function $q$ is unknown to the learner who cannot run an online algorithm on the sequence of losses~$(h_t)$ and recover a decision in the original space~$\X$ by inverting~$q$. Preconditioning by the Jacobian of $q$ (which is unknown) is not possible either. 
We consider two different settings depending on the feedback available to the learner: 
\begin{enumerate}[label=(\roman*),leftmargin=*]
\item \textit{Exact Gradient Feedback:} In this setting, we suppose that the loss function $\ell_t$ is differentiable and the learner has access to the gradient $\nabla \ell_t(x_t)$ of the loss function $\ell_t$ at each time step of the online interaction. 
\item \textit{Bandit Feedback:} In this setting, the learner has only access to a single function value call to the loss function $\ell_t$ at each round.  
\end{enumerate}

\noindent\textbf{Online Gradient Descent.} In this work, we will focus on the most famous and classical algorithm to address this problem: Online Gradient Descent (OGD) \cite{zinkevich03}. 
In the exact gradient feedback setting, the learner performs the following update for all time steps $t$:
\begin{equation}
x_{t+1} = \Pi_{\X}(x_t - \eta \nabla \ell_t(x_t))\,,
\tag{OGD}
\label{eq:ogd-main}
\end{equation}
where $\eta$ is a positive step size and $\Pi_{\X}$ is the Euclidean projection over the convex compact set~$\X$.  
In the bandit setting, the learner will use the one-point estimate of the gradient resulting in the bandit gradient descent algorithm 
proposed by \cite{flaxman-kalai-mcmahan05}. We will discuss this case in more detail in Section~\ref{sec:bandit-feedback}. 

We conclude this section by introducing some general notation that will be used throughout the paper. 

\noindent\textbf{Additional notation.} 
For a nonzero integer $d$, we denote by $[d] :=\{1, \dots, d\}\,.$
For any function $f: \reals^d \rightarrow \reals^d$, we denote by $J_f(x) \in \reals^{d \times d}$ the Jacobian of~$f$ at $x$. 
Given a strictly convex and differentiable function $R:\reals^d \rightarrow \reals$, the Bregman divergence induced by $R$ is defined as
$D_R(x\| y) \eqdef R(x) - R(y) - \nabla R(y)^{\top} (x-y)~.$  
For a closed convex set $\K$ and a strictly convex regularizer $R$, we use $\Pi^R_{\K}(x) \eqdef \argmin_{y \in \K} D_R(y\| x)$ to denote the Bregman projection, 
and we use the shorthand notation $\Pi_{\K} \eqdef \Pi^{\|\cdot \|^2}_{\K}$ for the Euclidean projection.  
Given a positive-definite matrix $M \in \reals^{d \times d}$, we define the norm $\|x\|_M \eqdef \sqrt{x^{\top}Mx}$ and we denote by $M^{-\top}$ the inverse transpose matrix.   
We use the notation $B_p \eqdef \{x \in \reals^d: \|x\|_p \le 1\}$ for an $\ell_p$ ball and 
$B^{+}_p$ to denote the intersection of the $\ell_p$ ball and the positive orthant. 
Differentiability on a compact domain is understood through
extensions of the stated regularity to an open neighborhood.
Bounds on derivatives are in the induced Euclidean operator norms.

\section{Regret Analysis of OGD for OHCO under Exact Gradient Feedback} 
\label{sec:full-info}

In this section, we establish a $\mathcal{O}(\sqrt{T})$ regret bound for OGD in the OHCO setting described in Section~\ref{sec:OHCO} 
under exact gradient feedback, i.e. when the learner has access to gradients of the loss functions $\ell_t$. 
We start by introducing our main assumptions in Section~\ref{subsec:full-info-as} before stating our first main result in Section~\ref{subsec:full-info-reg-bound} and providing a proof sketch in Section~\ref{subsec:full-info-proof-sketch}.  

\subsection{Assumptions}
\label{subsec:full-info-as}

We consider the same set of assumptions on the reparameterization and the loss functions as in \cite{ghai-lu-hazan22}.  

\begin{assumption}[Hessian-compatible reparameterization]
\label{as:reparam}
The map $q:\X\to\Y$ is a smooth bijection onto the convex set $\Y$ and there exists a twice continuously differentiable, strictly convex regularizer 
$R:\Y\to\mathbb R$ such that, for all $x\in\X$, with $z=q(x)$, 
$\big[\nabla^2 R(z)\big]^{-1} = J_q(x)J_q(x)^\top\,.$ 
\end{assumption}

Assumption~\ref{as:reparam} is a crucial assumption which restricts the class of parameterization functions $q$. 
Examples of such parameterizations include the quadratic reparameterization with the entropy regularizer,  
the exponential reparameterization with the log-barrier regularizer and the power reparameterization with 
tempered Bregman divergences \citep{amid-et-al19robust-bitempered}. 
We refer the reader to Section~3.1 in~\cite{ghai-lu-hazan22} for a more detailed discussion on these examples. 
We will elaborate more on Assumption~\ref{as:reparam} and its usefulness in Section~\ref{subsec:full-info-proof-sketch} when discussing the proof of our main results in the exact gradient feedback setting, and in Section~\ref{sec:discussion-hessian-comp-as} where we will discuss both verifiable necessary and sufficient conditions under which it is satisfied, and its necessity. 

\begin{assumption}[Regularizer and reparameterization smoothness] 
\label{as:reg-R-reparam-q-smoothness}
There exists \(G>1\) such that:
\begin{enumerate}[label=(\roman*),leftmargin=*]
    \item The map \(q\) is \(G\)-Lipschitz on \(\X\), and the 1st and 2nd derivatives of \(q^{-1}\) are bounded by \(G\) on \(\Y\).
    \item The regularizer $R$ is $1$-strongly convex and smooth with 1st and 3rd derivatives bounded by~$G$.
    \item For every \(z\in\Y\), the map \(w\mapsto D_R(z\|w)\) is \(G\)-Lipschitz on \(\Y\).
\end{enumerate}
\end{assumption}

Our last assumption asks that the loss functions $\ell_t$ have uniformly bounded gradients and that the diameter of $\X$ is bounded. 
These assumptions are also standard in OCO. 

\begin{assumption}[Boundedness of gradients and diameter]
\label{as:boundedness-grad-bregman}
There exists $G_F > 0$ s.t. $\|\nabla \ell_t(x)\|_2 \le G_F$ for all $x\in \X\,.$
Moreover, there exists~$D_1 > 0$ such that $\sup_{z,z' \in \Y} D_R(z\| z') \le D_1$\,.
\end{assumption}

\subsection{Regret Bound}
\label{subsec:full-info-reg-bound}

We are now ready to state the main result of this section.

\begin{tcolorbox}[colframe=white!, top=2pt,left=2pt,right=2pt,bottom=2pt] 
\begin{theorem}[OGD regret under full-information feedback]
\label{thm:main}
Let Assumptions~\ref{as:reparam}-\ref{as:boundedness-grad-bregman} hold
and let $T\ge1$. Define $A:=G_F^2\left(5G^6+\frac12G^2\right).$
Setting $\eta=\sqrt{\frac{D_1}{AT}}$, the regret of \eqref{eq:ogd-main} on the sequence of hidden-convex losses $(\ell_t)$ defined in~\eqref{eq:hidden-convex} using stepsize $\eta$ is bounded as follows:  
\[
R_T
\le
2\sqrt{AD_1T}
+
\frac{G^6G_F^3D_1}{A}
=
\mathcal{O}(\sqrt T).
\]
\end{theorem}
\end{tcolorbox}

This regret bound improves over the $\mathcal{O}(T^{2/3})$ regret bound established in \cite{ghai-lu-hazan22}. The leading $\sqrt T$ coefficient is $O(G^3G_F)$, where $G_F$
bounds the original loss gradients. For comparison, the bound
in~\cite{ghai-lu-hazan22} has dependence
$O(G^{10/3}G_{F,h})$, where $G_{F,h}$ bounds the hidden loss
gradients (under our assumptions, $G_{F,h}\le GG_F$). 

\subsection{Proof Sketch of Theorem~\ref{thm:main}}
\label{subsec:full-info-proof-sketch}

In this section, we provide a proof sketch of Theorem~\ref{thm:main}. We provide a complete proof in Section~\ref{sec:analysis}. 
The main strategy of the proof is inspired by \cite{ghai-lu-hazan22} which shows that OGD applied to non-convex functions is an approximation of online mirror descent
applied to convex functions under reparameterization. We briefly recall the intuition behind 
this idea for the convenience of the reader to motivate Assumption~\ref{as:reparam} before discussing our new key technical steps 
driving our tighter analysis. 

\noindent\textbf{Implicit OMD via OGD reparameterization.} In continuous time, 
Amid and Warmuth \cite{amid2020reparameterizing} showed that 
the mirror-flow trajectory in the hidden space and the gradient-flow trajectory in the original space coincide exactly under the reparameterization. 
More precisely, if $z(t)=q(x(t))$, they showed that both trajectories of the following ordinary differential equations: 
\begin{equation}
\label{eq:contin-flows}
\frac{d}{dt}\nabla R(z(t)) = -\nabla h(z(t))\,, 
\quad \frac{dx(t)}{dt} = -\nabla (h\circ q)(x(t)) 
\end{equation}
coincide if the mirror descent regularizer $R$ 
and the reparameterization function $q$ satisfy $[\nabla^2R(q(x))]^{-1} = J_q(x)J_q(x)^\top$. 
In discrete time, OGD iterates on the losses $(\ell_t)$ are given by~\eqref{eq:ogd-main} and Online Mirror Descent (OMD) iterates on the convex losses $(h_t)$ are defined as follows: 
\begin{equation}
y_{t+1} =
\arg\min_{y\in\Y}
    \left\{
        \langle \nabla h_t(y_t),y-y_t\rangle
        +
        \frac1\eta D_R(y\|y_t)
    \right\}\,,
    \tag{OMD}
    \label{eq:omd-main}
\end{equation}
with $y_0 = q(x_0)\,.$ 
In discrete time, the exact equivalence in the continuous dynamics in \eqref{eq:contin-flows} no longer holds because of the discretization error. 
Starting from a coupled state $y_t=q(x_t)$, one step of OGD on $(\ell_t)$ in the original space produces \(x_{t+1}\) and hence $z_{t+1}:=q(x_{t+1})$, while one step of OMD on $(h_t)$ in the hidden space produces~$y_{t+1}$. 
Ghai, Lu, and Hazan \cite{ghai-lu-hazan22} showed that these two hidden-space iterates are \(O(\eta^{3/2})\)-close. Throughout the coupling argument, the ghost OMD step is initialized
anew at $z_t=q(x_t)$ on each round; it is not a separately evolving
OMD trajectory.

\noindent\textbf{Our key improvement.} To obtain Theorem~\ref{thm:main}, our key contribution is to show a sharper $O(\eta^{2})$ estimate of the distance between two hidden-space iterates, which is small enough to preserve the classical $\sqrt T$ regret rate after applying a perturbed-OMD regret bound.

\begin{lemma}[One-step OGD-OMD coupling]
\label{prop:dist-minimizers-eta2-main}
Suppose Assumptions~\ref{as:reparam}-\ref{as:boundedness-grad-bregman} hold.
If $y_t=z_t=q(x_t)$, $x_{t+1}$ is computed via one step of
\eqref{eq:ogd-main}, and $y_{t+1}$ via one step of \eqref{eq:omd-main}, then
\begin{equation*}
    \|z_{t+1}-y_{t+1}\|_2
\le C_\eta := G^5(5G_F^2+G_F^3\eta)\eta^2.
\end{equation*}
\end{lemma}

The proof of Lemma~\ref{prop:dist-minimizers-eta2-main} is the main new step. Starting from the coupled state \(y_t=z_t=q(x_t)\), we first localize both
\(y_{t+1}\) and \(z_{t+1}\) in an \(O(\eta)\)-neighborhood \(\Y_t\) of \(y_t\).
On this local set, both updates can be expressed as minimizers of the same
strongly convex quadratic model, up to different higher-order perturbations:
\begin{align*}
y_{t+1}
&=
\arg\min_{y\in \Y_t}
    \left\{
        \Phi_t(y)+\varepsilon_t^{\rm OMD}(y)
    \right\}, \qquad
z_{t+1}
=
\arg\min_{y\in \Y_t}
    \left\{
        \Phi_t(y)+\varepsilon_t^{\rm OGD}(y)
    \right\},
\end{align*}
\begin{equation*}
\Phi_t(y) := \langle \nabla h_t(y_t),y-y_t\rangle
+ \frac{1}{2\eta}\|y-y_t\|_{\nabla^2R(y_t)}^2.
\end{equation*}
Here \(\varepsilon_t^{\rm OMD}\) is the Taylor remainder obtained by replacing
\(D_R(y\|y_t)\) with its second-order expansion at \(y_t\), while
\(\varepsilon_t^{\rm OGD}\) is the Taylor remainder obtained after changing
variables \(y=q(x)\) in the OGD update and expanding \(q^{-1}\) at \(y_t\) 
(see Lemma~\ref{lem:omd-rewritten} and Lemma~\ref{lem:ogd-rewritten} in App.~\ref{appx:proof-main-lemma}
for their explicit definitions). 
The leading quadratic terms agree exactly because Assumption~\ref{as:reparam}
identifies the OGD Euclidean metric in \(x\)-space with the mirror metric
\(\nabla^2R(y_t)\) in \(y\)-space.

The crucial departure from \cite{ghai-lu-hazan22} is how we compare these two
perturbed minimization problems. Rather than controlling the sup-norm difference
between the perturbed objectives and then translating this into a distance between
minimizers, we compare the first-order optimality conditions directly. Since
\(\Phi_t\) is \(1/\eta\)-strongly convex and the normal cone of \(\Y_t\) is
monotone, Lemma~\ref{lem:dist-minimizers-eta-times-norm} in App.~\ref{appx:proof-main-lemma} gives
\[
    \|y_{t+1}-z_{t+1}\|
    \le
    \eta
    \left\|
        \nabla \varepsilon_t^{\rm OGD}(z_{t+1})
        -
        \nabla \varepsilon_t^{\rm OMD}(y_{t+1})
    \right\|.
\]
Since both iterates lie in an $\mathcal{O}(\eta)$-neighborhood of $y_t$, we prove that $\|\nabla \varepsilon_t^{\rm OGD}(z_{t+1})\| = \mathcal{O}(\eta)$ and $\|\nabla \varepsilon_t^{\rm OMD}(y_{t+1})\|= \mathcal{O}(\eta)$ (in Lemma~\ref{lem:norm-grad-eps-OMD}-\ref{lem:norm-grad-eps-OGD} in App.~\ref{appx:proof-main-lemma}). Hence $\|y_{t+1}-z_{t+1}\| = \mathcal{O}(\eta^2)\,.$ 
This sharper one-step coupling improves the $\mathcal{O}(\eta^{3/2})$ discrepancy
used in \cite{ghai-lu-hazan22} to $\mathcal{O}(\eta^2)$.
Finally, the mapped OGD iterates $z_t=q(x_t)$ can be viewed as
perturbed OMD iterates, with each ghost step initialized at the
current point $z_t$.
The perturbation errors $r_{t+1}=z_{t+1}-y_{t+1}$ satisfy
$\|r_{t+1}\|_2\le C_\eta:=G^5(5G_F^2+G_F^3\eta)\eta^2.$ 
Applying the perturbed-OMD regret bound
(Lemma~\ref{lem2}) yields: 
\[
R_T
\le
\frac{D_1}{\eta}
+
\frac{\eta G^2G_F^2T}{2}
+
\frac{GC_\eta T}{\eta}
=
\frac{D_1}{\eta}+A\eta T+G^6G_F^3\eta^2T.
\]
The choice $\eta=\sqrt{D_1/(AT)}$ gives
Theorem~\ref{thm:main}.
See App.~\ref{sec:analysis} for the full proof. 

\section{Discussion of the Hessian-Compatibility Hidden Geometry Assumption~\ref{as:reparam}}
\label{sec:discussion-hessian-comp-as}

Our analysis relies on Assumption~\ref{as:reparam} which allows 
to obtain an implicit OMD reparameterization and exploit the connection with OMD analysis over convex functions 
to obtain our tight regret analysis.  
In this section, we discuss this assumption in more depth. First, we answer an open question 
in \cite{ghai-lu-hazan22} regarding whether we can relax their sufficient diagonal Jacobian condition (Assumption 4 therein) 
to verify Assumption~\ref{as:reparam}. In section~\ref{subsec:relaxing-diag-jacobian}, we provide necessary and sufficient conditions under which Assumption~\ref{as:reparam} is satisfied. 
Second, we address the question of the necessity of this Hessian-compatibility hidden-geometry assumption, which is also mentioned as an open question in 
\cite{ghai-lu-hazan22}. In Section~\ref{subsec:lin-OGD-regret-without-Hessian-comp}, we show that Assumption~\ref{as:reparam} cannot in general be removed from a uniform sublinear-regret guarantee for Euclidean OGD. 

\subsection{Relaxation of the diagonal Jacobian assumption in \cite{ghai-lu-hazan22}}
\label{subsec:relaxing-diag-jacobian}

Define the inverse metric field $M:\Y\to\mathbb R^{d\times d}$ by
\begin{equation}
\label{eq:matrix-field-M}
    M(z) := \left[ J_q(q^{-1}(z))J_q(q^{-1}(z))^\top \right]^{-1}.
\end{equation}
Recall that Assumption~\ref{as:reparam} states that there exists a scalar function $R:\Y\to\mathbb R$ s.t. $M(z)=\nabla^2 R(z),$ for all $z\in\Y.$ 
Theorem 7 in \cite{ghai-lu-hazan22} showed that if $J_q$ is diagonal then there exists such $R$. They used this diagonal separable structure  
to construct a regularizer $R$ explicitly as a solution to second-order ODEs with coefficients depending on the parameterization function $q$ and its derivatives. 
In the following result, we relax this condition and provide more general conditions under which the matrix field $M$ is a Hessian field. 

\begin{tcolorbox}[colframe=white!, top=2pt,left=2pt,right=2pt,bottom=2pt] 
\begin{proposition}
\label{prop:matrix-field-eq-Hessian}
Let $q:\X\to\Y$ be a $C^3$ diffeomorphism, with $\Y$ convex, and let
$M$ be defined by~\eqref{eq:matrix-field-M}. Then there exists a
$C^3$ scalar function $R$ such that
$M=\nabla^2R$ if and only if
\[
\partial_{z_k}M_{ij}(z)
=
\partial_{z_j}M_{ik}(z),
\qquad
\forall i,j,k\in[d],\quad z\in\Y.
\]
Such $R$ is unique up to an affine function. Moreover, since $M(z)\succ0$,
$R$ is strictly convex; on the compact domain $\Y$ of our OHCO setting,
$R$ is uniformly strongly convex.
\end{proposition}
\end{tcolorbox}

This result follows as a consequence of the fundamental result from differential analysis stating that a vector field with 
matching cross-partial coordinate function derivatives is a conservative field, i.e., a gradient field deriving from a scalar potential.  
See Proposition~\ref{prop:existence-grad-potential} for a precise statement of that result and a complete proof of Proposition~\ref{prop:matrix-field-eq-Hessian}. 

We provide two examples where the Jacobian of the parameterization $q$ is not diagonal in general, yet the matrix field $M$ defined in \eqref{eq:matrix-field-M} satisfies the cross-derivatives conditions $\partial_{z_k}M_{ij}=\partial_{z_j}M_{ik}$ of Proposition~\ref{prop:matrix-field-eq-Hessian}.  
Moreover, the compatibility Assumption~\ref{as:reparam} is satisfied with explicit regularizers $R$. 
Detailed derivations for both examples can be found in App.~\ref{appx:detailed-examples}. 

\begin{example}[Affine mixing of a separable parameterization]
Let $q(x)=As(x)+b$ and $s(x)=\bigl(s_1(x_1),\dots,s_d(x_d)\bigr)$ where \(A\in\mathbb R^{d\times d}\) is invertible, \(b\in\mathbb R^d\), and each $s_i$ is a $C^2$ diffeomorphism from an open interval
onto its image. Then $J_q(x)=A\,D(x)$ with $D(x):=\operatorname{diag}\bigl(s_1'(x_1),\dots,s_d'(x_d)\bigr)$. 
Hence $M(z) = A^{-T}D(x)^{-2}A^{-1}\,,z=q(x).$ Moreover, $M$ is explicitly the Hessian of the following regularizer: $R(z):=\sum_{i=1}^d r_i\bigl((A^{-1}(z-b))_i\bigr)$ and $r_i''(t) = \bigl(s_i'(s_i^{-1}(t))\bigr)^{-2}.$ 
For example, if $s_i(u_i) = \exp(u_i), i \in [d],$ $r_i(t) = - \log t$ satisfies the conditions. 
\end{example}

\begin{example}[Rank-one nonlinear mixing]
Let \(a\in\mathbb R^d\) and let $q(x)=x+h(a^\top x)\,a$ 
where \(h:\mathbb R\to\mathbb R\) is \(C^2\). We have $J_q(x)=I+h'(a^\top x)\,aa^\top.$
This Jacobian is typically non-diagonal whenever at least two components of $a$ are nonzero.
\end{example}

Both examples also satisfy Assumption~\ref{as:reg-R-reparam-q-smoothness}
on suitable compact domains; we give the corresponding domain conditions
and verification in App.~\ref{appx:detailed-examples}.

These examples enlarge the class of reparameterizations certified in the
coordinates in which the problem is presented, but we do not claim that
they are intrinsically nonseparable. In particular, the affine mixing in
Example~1 can sometimes be absorbed into the hidden loss, while Example~2
becomes coordinatewise after a rotation aligning $a$ with a coordinate axis.
Identifying standard applications that are provably irreducibly nonseparable
remains open.

\subsection{OGD can incur linear regret without Hessian-compatible hidden geometry}
\label{subsec:lin-OGD-regret-without-Hessian-comp}

 In this section, we discuss how crucial is Assumption~\ref{as:reparam} to obtain sublinear regret for OGD 
in our nonconvex online learning setting. Perhaps surprisingly, we show that OGD can incur linear regret 
if the Hessian compatibility assumption does not hold; i.e., if the metric induced by the symmetric positive definite matrix $M$ 
in \eqref{eq:matrix-field-M} does not derive from a Hessian metric such that $M(z) = \nabla^2 R(z)$ for some regularization function $R$. 

\begin{tcolorbox}[colframe=white!, top=2pt,left=2pt,right=2pt,bottom=2pt]
\begin{theorem}[Linear regret without Hessian-compatible geometry]
\label{thm:linear-reg-lower-bound}
There exist a smooth diffeomorphism $q:\X\to\Y$ with
$\X=\Y=[0,1]^2$ and $J_q(x)J_q(x)^\top\succeq\mu I$ for some
$\mu>0$, whose inverse metric field is not the Hessian of any regularizer,
an initialization $x_1\in\operatorname{int}\X$, and constants
$c,C,S,\eta_0>0$ with the following property.
For every $T\ge1$ and every constant stepsize $0<\eta\le\eta_0$,
there exists an adaptive sequence of linear hidden losses
$h_t(z)=\langle s_t,z\rangle$, with $\|s_t\|_2\le S$, such that
projected OGD on $\ell_t=h_t\circ q$ satisfies
$R_T\ge cT-\frac{C}{\eta}.$ 
All constants, $q$, and $x_1$ are independent of $T$ and $\eta$.
In particular, choosing $\eta=\Theta(T^{-1/2})$, held constant during
the run, gives $R_T=\Omega(T)$ for sufficiently large $T$.
\end{theorem}
\end{tcolorbox}

The theorem shows that smoothness and uniform positive definiteness of the induced hidden geometry are not sufficient for sublinear regret of OGD. This means that Assumption~\ref{as:reparam} is not merely a technical condition for our analysis of OGD in Theorem~\ref{thm:main}. 
The obstruction is geometric: when the metric induced by the reparameterization is not Hessian-compatible, OGD may fail to achieve sublinear regret even when the reparameterization is smooth and its induced metric is uniformly positive definite.
The hidden-space dynamics of OGD need not arise from an integrable potential. 
An adversary can then exploit this fact by forcing the iterates to cycle around a loop with nonzero circulation and force linear regret. We elaborate more on this in the next proof sketch.

\noindent\textbf{Proof sketch of
Theorem~\ref{thm:linear-reg-lower-bound}.}
We provide a proof sketch to highlight the intuitions.
The full proof is provided in
App.~\ref{app:proof-thm-linear-reg-lower-bound}.
The construction uses the two-dimensional reparameterization
$q(x_1,x_2)
=
\bigl(x_1,\,x_2-x_1(1-x_1)x_2(1-x_2)\bigr)\,, \X=\Y=[0,1]^2.$ 
For each fixed $x_1$, the second coordinate is strictly increasing
in $x_2$ and maps $[0,1]$ onto itself. Thus $q$ is a smooth
diffeomorphism between convex compact domains, with a uniformly
positive definite induced metric. Its inverse metric $M$, defined
in~\eqref{eq:matrix-field-M}, is not Hessian-compatible, as reflected
in the nonzero circulation below. 

We consider linear hidden losses
$h_t(z)=\langle s_t,z\rangle$ and original losses
$\ell_t(x)=h_t(q(x))$, where the adversary chooses $s_t$.
Writing $z_t=q(x_t)$, we have
$\nabla\ell_t(x_t)=J_q(x_t)^\top s_t$.
Whenever projection is inactive, a Taylor expansion of $q$ shows that
OGD in the original coordinates induces the hidden-space dynamics
\begin{equation}
\label{eq:ogd-dyn-hidden-space-main}
z_{t+1}
=
z_t-\eta M(z_t)^{-1}s_t+O(\eta^2),
\end{equation}
for uniformly bounded $s_t$.
This suggests choosing $s_t=-M(z_t)v_t$ to cancel the induced metric
and move the hidden iterate approximately in a prescribed direction
$v_t$:
$z_{t+1}=z_t+\eta v_t+O(\eta^2)$.
However, these local errors need not cancel after a complete circuit.
To repeat the same circuit without accumulating drift, we implement
this steering idea exactly rather than only to first order.

Specifically, suppose that the desired next hidden point is
$z_t^+=z_t+\Delta v_t\in\operatorname{int}\Y$, where
$v_t\in\{\pm e_1,\pm e_2\}$ and $0<\Delta\le\eta$.
The corresponding target in the original space is $q^{-1}(z_t^+)$.
We therefore choose $s_t$ by solving
$x_t-\eta J_q(x_t)^\top s_t=q^{-1}(z_t^+)$, which gives
$s_t =
\frac1\eta J_q(x_t)^{-\top}
\bigl[x_t-q^{-1}(z_t^+)\bigr].$ 
With this choice, the OGD pre-projection point is exactly
$q^{-1}(z_t^+)\in\operatorname{int}\X$. Hence projection is inactive and 
$z_{t+1}=z_t+\Delta v_t$ holds exactly.
Only the adversary uses $q$; the learner still performs ordinary OGD.
Moreover, the vectors $s_t$ are uniformly bounded because
$\Delta\le\eta$ and $q^{-1}$ has bounded derivatives on $\Y$.
This exact choice retains the same inverse-metric cancellation
as its first-order counterpart. Indeed, Taylor expansion of $q^{-1}$
gives
$s_t
=
-\frac{\Delta}{\eta}M(z_t)v_t
+
O(\Delta^2/\eta).$ 
Fixing a comparator $u\in\X$ and writing $z_u=q(u)$, define
$F_u(z):=M(z)(z-z_u)$.
The instantaneous regret against $u$ can then be written as
\begin{equation*}
\ell_t(x_t)-\ell_t(u)
=
\langle s_t,z_t-z_u\rangle 
=
-\frac1\eta
\langle F_u(z_t),z_{t+1}-z_t\rangle
+O(\Delta^2/\eta).
\end{equation*}
Hence cumulative regret along the prescribed trajectory is,
up to controlled error terms, a rescaled discrete line integral
of the vector field $F_u$. 
We now choose a closed path along which this integral is negative.
For $u=(0,1/2)$, we have $z_u=(0,1/2)$, and a direct calculation at
$z^0=q(1/2,1/2)=(1/2,7/16)$ gives
$\partial_{z_1}(F_u)_2(z^0)
-
\partial_{z_2}(F_u)_1(z^0)
=
-\frac14.$ 
By continuity, there is a small square
$\mathcal Q\subset\operatorname{int}\Y$ on which the curl is
uniformly negative and Green's theorem therefore gives a negative
counterclockwise circulation
$I_{\mathcal Q}
:=
\oint_{\partial\mathcal Q}F_u(z)\cdot dz
<0.$

Let $L$ be the side length of $\mathcal Q$.
Set $n=\lceil L/\eta\rceil$ and $\Delta=L/n$, initialize $z_1$ at its
lower-left corner, and prescribe $n$ steps in each of the directions
$e_1,e_2,-e_1,-e_2$, in that order.
The hidden iterates then traverse the boundary counterclockwise and
return exactly to the starting point after $N=4n=\Theta(1/\eta)$
rounds. The Riemann-sum error is $O(\Delta)$, and the cumulative
Taylor error in the regret is $O(N\Delta^2/\eta)=O(1)$.
Consequently, one complete cycle satisfies
\[
\sum_{t=1}^{N}\bigl(\ell_t(x_t)-\ell_t(u)\bigr)
=
-\frac{I_{\mathcal Q}}{\eta}+O(1)
=
\Omega(1/\eta)
\]
for sufficiently small $\eta$.
Thus each cycle contributes a positive constant amount of regret
per round, with constants independent of $\eta$. 
Because the trajectory closes exactly, the same cycle can be
repeated without drift. Completing as many cycles as fit within
the horizon $T$ and setting all remaining losses to zero yields
$R_T\ge cT-C/\eta$ for constants $c,C>0$ independent of $T$ and $\eta$,
since external regret is at least regret against the fixed
comparator $u$. In particular, a constant stepsize within the run
of order $\eta=\Theta(T^{-1/2})$ gives $R_T=\Omega(T)$.  

\begin{remark}
Theorem~\ref{thm:linear-reg-lower-bound} is an existential lower bound
for Euclidean OGD: it does not assert that every incompatible
reparameterization leads to linear regret.
More generally, the same exact steering construction applies to a smooth
diffeomorphism between compact convex domains in $\mathbb R^2$ whenever,
for some comparator $u$, the field $M(z)(z-q(u))$ has nonzero curl
at an interior point. A sufficiently small square and an appropriate
orientation then give a negative circulation.
\end{remark}

\noindent\textbf{Open question.} The lower bound of Theorem~\ref{thm:linear-reg-lower-bound} only holds for OGD. 
Whether Assumption~\ref{as:reparam} is necessary to obtain sublinear regret in general, for \textit{any} online algorithm, remains open.

\section{Regret Analysis of OGD for OHCO under Bandit Feedback} 
\label{sec:bandit-feedback}

\subsection{Expected Regret Bound in the Bandit Setting}

We now consider the one-point bandit feedback, where the learner observes only one function value per round. 
We analyze the classical Bandit OGD (BOGD) algorithm with spherical smoothing (Algorithm~\ref{alg4}, \cite{flaxman-kalai-mcmahan05}) in the hidden-convex setting, against an oblivious adversary. 
Let $U_{sp}$ denote the uniform distribution over the Euclidean unit sphere. 
For an exploration radius $\delta\in(0,1)$, the algorithm maintains its iterates in the shrunken domain $\X_\delta := (1-\delta)\X$, so that the perturbed query points remain feasible (under the standard normalization $0\in\X$ and $B_2\subseteq\X$ as in \cite{flaxman-kalai-mcmahan05}). 
We denote by $\Y_\delta := q(\X_\delta)$ the corresponding hidden-space domain, which is only used in the analysis to compare the mapped BOGD iterates with mirror descent on the convex losses.

\begin{wrapfigure}{r}{0.45\linewidth}
\vspace{-2.7\baselineskip}
\begin{minipage}{\linewidth}
\begin{algorithm}[H]
\caption{Bandit OGD (BOGD) \cite{flaxman-kalai-mcmahan05}} \label{alg4}
\begin{algorithmic}[1]
\STATE Input: $x_1 \in \X_\delta$\\
\FOR{$t = 1, \ldots, T$}
\STATE Sample a random direction $\zeta_t \sim U_{sp}$
\STATE Choose $\hat{x}_t = {\color{blue} x_t} + \delta \zeta_t$
\STATE Observe the bandit loss $\ell_t(\hat{x}_t)$
\STATE Construct ${\color{MyGreen} g_t} = \frac{d}{\delta} \ell_t(\hat{x}_t) \zeta_t$
\STATE Update 
$x_{t+1} = \Pi_{\X_{\delta}} ({\color{blue} x_t}-\eta {\color{MyGreen} g_t})$
\ENDFOR
\end{algorithmic}
\end{algorithm}
\end{minipage}
\vspace{-2.4\baselineskip}
\end{wrapfigure}

Before stating the result, we introduce a few extra standard assumptions on the smoothness and boundedness of the losses $(\ell_t)$, and the diameter of $\X$. 

\begin{assumption}[Loss boundedness]
There exists $M>0$ s.t. $|\ell_t(x)| \le M$, for all $x \in \X, t \geq 1$. 
\end{assumption}
\begin{assumption}[Smoothness]
\label{as:smoothness-bandit}
    For any $t$, the loss function $\ell_t$ is smooth, i.e., 
    there exists a constant $H>0$, such that for any $x,x' \in \X$ and any $t \geq 1$, 
    $\|\nabla \ell_t(x) - \nabla \ell_t(x')\| \leq H \|x - x'\|.$ 
\end{assumption}

\begin{assumption}[Feasible set regularity]
\label{as:bounded-euclidean-diam}
There exists $D>0$ s.t. 
$\|x\|_2\le D, 
\|x-x'\|_2\le D$ for all $x,x'\in\X$. 
Moreover, $0\in\X$, $B_2\subseteq\X$, and the hidden shrunken domain
$\Y_\delta:=q((1-\delta)\X)$ is convex for the exploration radius
$\delta\in(0,1)$ used by Algorithm~\ref{alg4}.
\end{assumption}

In the following theorem, we present the no-regret guarantees of Bandit OGD.
\begin{tcolorbox}[colframe=white!, top=2pt,left=2pt,right=2pt,bottom=1pt] 
\begin{theorem}[Expected regret under bandit feedback]
\label{thm:bandit}
Fix an oblivious sequence of losses and a horizon $T\ge1$.
Let Assumptions~\ref{as:reparam}--\ref{as:bounded-euclidean-diam}
hold for the exploration radius specified below. 
Let Assumptions~\ref{as:reparam} to \ref{as:bounded-euclidean-diam} hold. 
Let $K_1:=G_FD+G^2HD+G_F, K_2:=6d^2G^6M^2, K_3:=d^3G^6M^3$. 
Then, the expected regret of BOGD (Algorithm \ref{alg4}) is bounded as follows: 
\begin{equation*}
    \mathbb E[R_T]
\le
C_{\rm band}\,T^{3/4}, \quad 
C_{\rm band} := 
2\sqrt2\,(K_1^2K_2D_1)^{1/4} +
K_1^{1/2} K_2^{-5/4} K_3D_1^{3/4}/\sqrt2\,,
\end{equation*}
when using the stepsize $\eta = \left(
\frac{4D_1^3}{K_1^2K_2T^3}
\right)^{1/4}$, and the radius $\delta = \left(
\frac{4K_2D_1}{K_1^2T}
\right)^{1/4}$ with $\delta \in(0,1)$. 
\end{theorem}
\end{tcolorbox}

Thus, under the stated assumptions, BOGD recovers the classical
$\mathcal O(T^{3/4})$ guarantee of one-point spherical-smoothing
OGD for convex losses~\citep{flaxman-kalai-mcmahan05}.

\subsection{Proof Sketch of Theorem~\ref{thm:bandit}}
\label{subsec:bandit-proof-sketch}

We give a proof sketch of Theorem~\ref{thm:bandit}; the complete proof is in
App.~\ref{appx:proofs-bandit}. The proof follows the same high-level strategy as in the exact gradient setting with a few similar steps to the analysis of OBGD in OCO: after mapping the BOGD iterates to the hidden
space, $z_t=q(x_t)$, we compare them to mirror-descent steps on the convex losses
$h_t$. The bandit setting introduces three additional errors: the domain must be
shrunk to allow perturbations, the one-point estimator is biased for the true
gradient, and the algorithm suffers the loss at the perturbed point
$\hat x_t=x_t+\delta\zeta_t$, rather than at $x_t$. To define the ghost OMD iterate in the hidden space (which is only used in the analysis and never computed), we introduce the hidden-space gradient estimator $\tilde g_t := J_q(x_t)^{-\top} g_t$ which is the natural hidden-space analogue of the bandit estimator $g_t$ by the chain rule: $\nabla h_t(z_t)=J_q(x_t)^{-\top}\nabla \ell_t(x_t), z_t=q(x_t).$
Note that it is not computed by the algorithm but introduced only to analyze the mapped BOGD trajectory as a perturbed OMD trajectory in $\Y_\delta=q(\X_\delta)$. Since \(|\ell_t|\le M\) and \(\|\zeta_t\|=1\), we have $\|g_t\| \leq dM/\delta$. 
Let $z^\star\in\arg\min_{z\in\Y}\sum_{t=1}^T h_t(z)$ and $z^\star_\delta\in\arg\min_{z\in\Y_\delta}\sum_{t=1}^T h_t(z)$ be two comparators, and let \(\hat z_t=q(\hat x_t)\). 
The starting point is the following regret decomposition:
\begin{multline*} 
\mathbb{E}[R_T]
\le
\underbrace{
\mathbb{E}\left[\sum_{t=1}^T
\langle \tilde g_t,z_t-z^\star_\delta\rangle\right]
}_{\text{\makecell{$R_T^{(1)}:$ perturbed OMD regret\\ for random linear losses}}}
+
\underbrace{
\sum_{t=1}^T h_t(z^\star_\delta)-h_t(z^\star)
}_{\text{\makecell{$R_T^{(2)}:$ domain\\ shrinkage error}}}
+
\underbrace{
\mathbb{E}\left[\sum_{t=1}^T
B_t\right]
}_{\text{\makecell{$R_T^{(3)}$: gradient\\ estimation bias}}}
+
\underbrace{
\mathbb{E}\left[\sum_{t=1}^T h_t(\hat z_t)-h_t(z_t)\right]
}_{R_T^{(4)}:\ \text{smoothing error}},
\end{multline*}
where $B_t := \langle b_t,z^\star_\delta-z_t\rangle$ and $b_t:=\mathbb{E}[\tilde g_t\mid\mathcal F_{t-1}]-\nabla h_t(z_t)\,.$ 
The four terms isolate the distinct sources of error in the bandit setting. The shrinkage term is controlled by comparing optimization over $\X$ and
$\X_\delta=(1-\delta)\X$, giving $R_T^{(2)}\le \delta G_F DT$. 
The smoothing term is controlled directly by Lipschitzness of $\ell_t$: $|R_T^{(4)}| \leq \delta G_F T.$ 
The bias term comes from the fact that the one-point spherical estimator is unbiased for the gradient of the ball-smoothed loss, not for $\nabla\ell_t(x_t)$. Using
smoothness of \(\ell_t\) and the chain-rule transformation above yields $\|b_t\|\le \delta GH$ and $\|z_t-z^\star_\delta\|\le GD$, and hence $|R_T^{(3)}|\le \delta G^2HDT\,.$ 
It remains to control \(R_T^{(1)}\), which is the term most closely related to the exact gradient setting proof. For each round, define the ghost OMD iterate, only used in the analysis: 
\begin{equation*}
    y_{t+1}
    =
    \arg\min_{y\in\Y_\delta}
    \left\{
        \langle \tilde g_t,y-z_t\rangle
        +
        \frac1\eta D_R(y\|z_t)
    \right\}.
\end{equation*}
The actual
mapped BOGD iterate satisfies $z_{t+1}=q(x_{t+1})=y_{t+1}+r_{t+1}$. 
As in the exact gradient feedback setting, the OGD--OMD one-step coupling applies pathwise because the proof only requires a bounded vector $g_t$, not an exact gradient.
With $\hat G_F=dM/\delta$, Corollary~\ref{cor:approx-lemma-bandit} gives $\|r_{t+1}\| \leq \hat C_\eta:= G^5\left(5\hat G_F^2+\hat G_F^3\eta\right)\eta^2.$ 
As in the exact-gradient setting, the perturbed OMD bound for the
random linear losses $\langle\tilde g_t,\cdot\rangle$ gives
$
R_T^{(1)}
\le
\frac{D_1}{\eta}
+
\frac{\eta T}{2}\tilde G_F^2
+
\frac{G\hat C_\eta T}{\eta}$,
$\tilde G_F:=G\hat G_F=\frac{GdM}{\delta}.$ 
Combining the four bounds yields: 
$\mathbb{E}[R_T]
    \le
    K_1\delta T
    +
    \frac{D_1}{\eta}
    +
    K_2\frac{\eta T}{\delta^2}
    +
    K_3\frac{\eta^2T}{\delta^3}\,.$
Balancing the first three terms gives the values of $\eta, \delta$ reported in Theorem~\ref{thm:bandit} and the desired
$\mathcal{O}(T^{3/4})$ bound follows.

\vspace{-1mm}
\section{Conclusion and Future Work}

In this paper, we studied online nonconvex learning for a structured class of hidden-convex losses. 
We showed that OGD achieves $\mathcal{O}(\sqrt{T})$ regret in the exact-gradient setting, sharpened the geometric conditions required for the underlying algorithmic equivalence, and proved that OGD can suffer linear regret when these conditions fail. 
We also extended the framework to one-point bandit feedback, establishing sublinear expected regret guarantees against oblivious adversaries.

Several questions remain open. 
First, our lower bound shows that OGD can suffer linear regret without Hessian compatibility, but it does not rule out sublinear regret for other algorithms. 
An important direction is therefore to determine whether online hidden-convex optimization admits sublinear, or even $\mathcal{O}(\sqrt{T})$, regret without this geometric assumption. 
Second, our bandit guarantees leave open the possibility of sharper rates, high-probability bounds, and algorithms matching the best-known guarantees from bandit online convex optimization. 
More broadly, it would be interesting to understand for which classes of hidden-convex functions online algorithms incur a regret that matches that of their convex counterparts.

\begin{ack}

The authors acknowledge the use of ChatGPT throughout the development of this work, particularly for constructing the instance for the linear regret lower bound of Theorem~\ref{thm:linear-reg-lower-bound}, developing and refining its proof, and formulating the examples in Section~\ref{subsec:relaxing-diag-jacobian}.
The manuscript was primarily written by the authors, with ChatGPT providing assistance with selected drafts and revisions.
The authors critically evaluated and revised the AI-assisted material,
checked the mathematical arguments, and assessed the novelty of the
results against the relevant literature.
The authors take full responsibility for the content of the manuscript.

Ioannis Panageas is supported by NSF grant CCF-2454115. This work is supported by the MOE Tier 2 Grant (MOE-T2EP20223-0018),  the CQT++ Core Research Funding Grant (SUTD) (RS-NRCQT-00002), the National Research Foundation Singapore and DSO National Laboratories under the AI Singapore Programme (Award Number: AISG2-RP-2020-016), and partially by Project MIS 5154714 of the National Recovery and Resilience Plan, Greece 2.0, funded by the European Union under the NextGenerationEU Program.
\end{ack}

\bibliographystyle{plain}
\bibliography{references}

@ARTICLE{ergen-pilanci25,
  author={Ergen, Tolga and Pilanci, Mert},
  journal={IEEE Transactions on Information Theory}, 
  title={The Convex Landscape of Neural Networks: Characterizing Global Optima and Stationary Points via Lasso Models}, 
  year={2025},
  volume={71},
  number={5},
  pages={3854-3870}}

@inproceedings{zhang-et-al20variational,
  title={Variational policy gradient method for reinforcement learning with general utilities},
  author={Zhang, Junyu and Koppel, Alec and Bedi, Amrit Singh and Szepesvari, Csaba and Wang, Mengdi},
  booktitle={Advances in Neural Information Processing Systems},
  year={2020}
}

@inproceedings{hazan-et-al19,
  title={Provably efficient maximum entropy exploration},
  author={Hazan, Elad and Kakade, Sham and Singh, Karan and Van Soest, Abby},
  booktitle={International Conference on Machine Learning},
  year={2019}
}

@inproceedings{zahavy-et-al21,
  title={Reward is enough for convex MDPs},
  author={Zahavy, Tom and O'Donoghue, Brendan and Desjardins, Guillaume and Singh, Satinder},
  booktitle={Advances in Neural Information Processing Systems},
  year={2021}
}

@InProceedings{barakat-et-al23rlgu,
  title = 	 {Reinforcement Learning with General Utilities: Simpler Variance Reduction and Large State-Action Space},
  author =       {Barakat, Anas and Fatkhullin, Ilyas and He, Niao},
  booktitle = 	 {International Conference on Machine Learning},
  year = 	 {2023}
}

@inproceedings{barakat-et-al25rlgu,
title={On the Global Optimality of Policy Gradient Methods in General Utility Reinforcement Learning},
author={Anas Barakat and Souradip Chakraborty and Peihong Yu and Pratap Tokekar and Amrit Singh Bedi},
booktitle={The Thirty-ninth Annual Conference on Neural Information Processing Systems},
year={2025}
}

@inproceedings{mladenovic-et-al22hidden,
title={Generalized Natural Gradient Flows in Hidden Convex-Concave Games and {GAN}s},
author={Andjela Mladenovic and Iosif Sakos and Gauthier Gidel and Georgios Piliouras},
booktitle={International Conference on Learning Representations},
year={2022}
}

@inproceedings{barakat-et-al26gumg,
title={Convex Markov Games and Beyond: New Proof of Existence, Characterization and Learning Algorithms for Nash Equilibria},
author={Anas Barakat and Ioannis Panageas and Antonios Varvitsiotis},
booktitle={International Conference on Artificial Intelligence and Statistics},
year={2026}
}

@InProceedings{kalogiannis-et-al25zs-cmg,
  title = 	 {Solving Zero-Sum Convex {M}arkov Games},
  author =       {Kalogiannis, Fivos and Vlatakis-Gkaragkounis, Emmanouil-Vasileios and Gemp, Ian and Piliouras, Georgios},
  booktitle = 	 {International Conference on Machine Learning},
  year = 	 {2025}
}

@InProceedings{gemp-et-al25cmg,
  title = 	 {Convex {M}arkov Games: A New Frontier for Multi-Agent Reinforcement Learning},
  author =       {Gemp, Ian and Haupt, Andreas Alexander and Marris, Luke and Liu, Siqi and Piliouras, Georgios},
  booktitle = 	 {International Conference on Machine Learning},
  year = 	 {2025}
}

@inproceedings{bhaskara-et-al25,
title={Descent with Misaligned Gradients and Applications to Hidden Convexity},
author={Aditya Bhaskara and Ashok Cutkosky and Ravi Kumar and Manish Purohit},
booktitle={International Conference on Learning Representations},
year={2025}
}

@article{gorissen-et-al26hidden-cvx,
  title={Hidden convexity in a class of optimization problems with bilinear terms},
  author={Gorissen, Bram L and den Hertog, Dick and Reusken, Meike},
  journal={Operations Research},
  volume={74},
  number={2},
  pages={1126--1152},
  year={2026},
  publisher={INFORMS}
}

@article{li-et-al05hidden-cvx-min,
  title={Hidden convex minimization},
  author={Li, Duan and Wu, Zhi-You and Joseph Lee, Heung-Wing and Yang, Xin-Min and Zhang, Lian-Sheng},
  journal={Journal of Global Optimization},
  volume={31},
  number={2},
  pages={211--233},
  year={2005},
  publisher={Springer}
}

@article{bental-teboulle96hidden-cvx,
  title={Hidden convexity in some nonconvex quadratically constrained quadratic programming},
  author={Ben-Tal, Aharon and Teboulle, Marc},
  journal={Mathematical Programming},
  volume={72},
  number={1},
  pages={51--63},
  year={1996},
  publisher={Springer}
}

@article{zeger-pilanci26,
  title={Unveiling Hidden Convexity in Deep Learning: A Sparse Signal Processing Perspective},
  author={Zeger, Emi and Pilanci, Mert},
  journal={arXiv preprint arXiv:2603.23831},
  year={2026}
}

@inproceedings{wang-et-al22iclr,
title={The Hidden Convex Optimization Landscape of Regularized Two-Layer Re{LU} Networks: an Exact Characterization of Optimal Solutions},
author={Yifei Wang and Jonathan Lacotte and Mert Pilanci},
booktitle={International Conference on Learning Representations},
year={2022},
url={https://openreview.net/forum?id=Z7Lk2cQEG8a}
}

@inproceedings{dorazio2025solving,
title={Solving hidden monotone variational inequalities with surrogate losses},
author={Ryan D'Orazio and Danilo Vucetic and Zichu Liu and Junhyung Lyle Kim and Ioannis Mitliagkas and Gauthier Gidel},
booktitle={The Thirteenth International Conference on Learning Representations},
year={2025},
url={https://openreview.net/forum?id=4ZX2a3OKEV}
}

@article{kalogiannis2024learning,
  title={Learning equilibria in adversarial team markov games: A nonconvex-hidden-concave min-max optimization problem},
  author={Kalogiannis, Fivos and Yan, Jingming and Panageas, Ioannis},
  journal={Advances in Neural Information Processing Systems},
  volume={37},
  pages={92832--92890},
  year={2024}
}

@inproceedings{patelsolving,
  title={Solving Neural Min-Max Games: The Role of Architecture, Initialization \& Dynamics},
  author={Patel, Deep and Vlatakis-Gkaragkounis, Emmanouil-Vasileios},
  booktitle={The Thirty-ninth Annual Conference on Neural Information Processing Systems}
}

@article{vlatakis2019poincare,
  title={Poincar{\'e} recurrence, cycles and spurious equilibria in gradient-descent-ascent for non-convex non-concave zero-sum games},
  author={Vlatakis-Gkaragkounis, Emmanouil-Vasileios and Flokas, Lampros and Piliouras, Georgios},
  journal={Advances in Neural Information Processing Systems},
  volume={32},
  year={2019}
}

@article{sakos2023exploiting,
  title={Exploiting hidden structures in non-convex games for convergence to Nash equilibrium},
  author={Sakos, Iosif and Vlatakis-Gkaragkounis, Emmanouil-Vasileios and Mertikopoulos, Panayotis and Piliouras, Georgios},
  journal={Advances in Neural Information Processing Systems},
  year={2023}
}

@article{hazan-et-al15beyond-cvx,
  title={Beyond convexity: Stochastic quasi-convex optimization},
  author={Hazan, Elad and Levy, Kfir and Shalev-Shwartz, Shai},
  journal={Advances in Neural Information Processing Systems},
  year={2015}
}

@inproceedings{gao-li-zhang18,
  title={Online learning with non-convex losses and non-stationary regret},
  author={Gao, Xiand and Li, Xiaobo and Zhang, Shuzhong},
  booktitle={International Conference on Artificial Intelligence and Statistics},
  year={2018}
  }

@inproceedings{zinkevich03,
  title={Online convex programming and generalized infinitesimal gradient ascent},
  author={Zinkevich, Martin},
  booktitle={International Conference on Machine Learning},
  pages={928--936},
  year={2003}
}

@article{li-et-al22implicit-bias,
  title={Implicit bias of gradient descent on reparametrized models: On equivalence to mirror descent},
  author={Li, Zhiyuan and Wang, Tianhao and Lee, Jason D and Arora, Sanjeev},
  journal={Advances in Neural Information Processing Systems},
  volume={35},
  pages={34626--34640},
  year={2022}
}

@article{levin-kileel-boumal25,
  title={The effect of smooth parametrizations on nonconvex optimization landscapes},
  author={Levin, Eitan and Kileel, Joe and Boumal, Nicolas},
  journal={Mathematical Programming},
  volume={209},
  number={1},
  pages={63--111},
  year={2025},
  publisher={Springer}
}

@article{lattimore24banditCO,
  title={Bandit convex optimisation},
  author={Lattimore, Tor},
  journal={arXiv preprint arXiv:2402.06535},
  year={2024}
}

@book{lattimore-szepesvari20,
  title={Bandit algorithms},
  author={Lattimore, Tor and Szepesv{\'a}ri, Csaba},
  year={2020},
  publisher={Cambridge University Press}
}

@inproceedings{maillard-munos10,
  title={Online learning in adversarial lipschitz environments},
  author={Maillard, Odalric-Ambrym and Munos, R{\'e}mi},
  booktitle={Joint european conference on machine learning and knowledge discovery in databases},
  pages={305--320},
  year={2010},
  organization={Springer}
}

@article{heliou-et-al20,
  title={Online non-convex optimization with imperfect feedback},
  author={H{\'e}liou, Am{\'e}lie and Martin, Matthieu and Mertikopoulos, Panayotis and Rahier, Thibaud},
  journal={Advances in Neural Information Processing Systems},
  volume={33},
  pages={17224--17235},
  year={2020}
}

@inproceedings{agarwal-gonen-hazan19,
  title={Learning in non-convex games with an optimization oracle},
  author={Agarwal, Naman and Gonen, Alon and Hazan, Elad},
  booktitle={Conference on Learning Theory},
  pages={18--29},
  year={2019},
  organization={PMLR}
}

@inproceedings{krichene-et-al15hedge-continuum,
  title={The hedge algorithm on a continuum},
  author={Krichene, Walid and Balandat, Maximilian and Tomlin, Claire and Bayen, Alexandre},
  booktitle={International Conference on Machine Learning},
  year={2015}
}

@inproceedings{suggala-netrapalli20,
  title={Online non-convex learning: Following the perturbed leader is optimal},
  author={Suggala, Arun Sai and Netrapalli, Praneeth},
  booktitle={Algorithmic Learning Theory},
  pages={845--861},
  year={2020},
  organization={PMLR}
}

@article{orabona19book-online-learning,
  title={A modern introduction to online learning},
  author={Orabona, Francesco},
  journal={arXiv preprint arXiv:1912.13213},
  year={2019}
}

@article{shalev-shwartz12oco-ftml,
  title={Online learning and online convex optimization},
  author={Shalev-Shwartz, Shai},
  journal={Foundations and Trends in Machine Learning},
  volume={4},
  number={2},
  pages={107--194},
  year={2012}
}

@article{hazan16oco,
  title={Introduction to online convex optimization},
  author={Hazan, Elad},
  journal={Foundations and Trends in Optimization},
  volume={2},
  number={3-4},
  pages={157--325},
  year={2016},
  publisher={Emerald Publishing Limited}
}

@article{chen-et-al25,
  title={Efficient algorithms for a class of stochastic hidden convex optimization and its applications in network revenue management},
  author={Chen, Xin and He, Niao and Hu, Yifan and Ye, Zikun},
  journal={Operations Research},
  volume={73},
  number={2},
  pages={704--719},
  year={2025},
  publisher={INFORMS}
}

@article{fatkhullin-et-al25constrained-hidden-cv,
  title={Global Solutions to Non-Convex Functional Constrained Problems with Hidden Convexity},
  author={Fatkhullin, Ilyas and He, Niao and Lan, Guanghui and Wolf, Florian},
  journal={arXiv preprint arXiv:2511.10626},
  year={2025}
}

@article{fatkhullin-he-hu25stoch-hidden-cv,
  title={Stochastic optimization under hidden convexity},
  author={Fatkhullin, Ilyas and He, Niao and Hu, Yifan},
  journal={SIAM Journal on Optimization},
  volume={35},
  number={4},
  pages={2544--2571},
  year={2025},
  publisher={SIAM}
}

@inproceedings{flaxman-kalai-mcmahan05,
author = {Flaxman, Abraham D. and Kalai, Adam Tauman and McMahan, H. Brendan},
title = {Online convex optimization in the bandit setting: gradient descent without a gradient},
year = {2005},
isbn = {0898715857},
publisher = {Society for Industrial and Applied Mathematics},
address = {USA},
booktitle = {Proceedings of the Sixteenth Annual ACM-SIAM Symposium on Discrete Algorithms},
pages = {385–394},
numpages = {10},
location = {Vancouver, British Columbia},
series = {SODA '05}
}

@article{zhang-et-al15, 
title={Online Bandit Learning for a Special Class of Non-Convex Losses}, volume={29}, 
url={https://ojs.aaai.org/index.php/AAAI/article/view/9589}, 
DOI={10.1609/aaai.v29i1.9589}, 
abstractNote={ &lt;p&gt; In online bandit learning, the learner aims to minimize a sequence of losses, while only observing the value of each loss at a single point. Although various algorithms and theories have been developed for online bandit learning, most of them are limited to convex losses. In this paper, we investigate the problem of online bandit learning with non-convex losses, and develop an efficient algorithm with formal theoretical guarantees. To be specific, we consider a class of losses which is a composition of a non-increasing scalar function and a linear function. This setting models a wide range of supervised learning applications such as online classification with a non-convex loss. Theoretical analysis shows that our algorithm achieves an O(poly(d)T2/3) regret bound when the variation of the loss function is small. To the best of our knowledge, this is the first work in online bandit learning that does not rely on convexity. &lt;/p&gt; }, 
number={1}, 
journal={Proceedings of the AAAI Conference on Artificial Intelligence}, author={Zhang, Lijun and Yang, Tianbao and Jin, Rong and Zhou, Zhi-Hua}, year={2015}, 
month={Feb.} 
}

@article{amid-et-al19robust-bitempered,
  title={Robust bi-tempered logistic loss based on bregman divergences},
  author={Amid, Ehsan and Warmuth, Manfred KK and Anil, Rohan and Koren, Tomer},
  journal={Advances in Neural Information Processing Systems},
  volume={32},
  year={2019}
}

@article{ghai-lu-hazan22,
  title={Non-convex online learning via algorithmic equivalence},
  author={Ghai, Udaya and Lu, Zhou and Hazan, Elad},
  journal={Advances in Neural Information Processing Systems},
  volume={35},
  pages={22161--22172},
  year={2022}
}

@article{amid2020reparameterizing,
  title={Reparameterizing mirror descent as gradient descent},
  author={Amid, Ehsan and Warmuth, Manfred KK},
  journal={Advances in Neural Information Processing Systems},
  volume={33},
  pages={8430--8439},
  year={2020}
}

@inproceedings{OGD,
author = {Zinkevich, Martin},
title = {Online Convex Programming and Generalized Infinitesimal Gradient Ascent},
year = {2003},
booktitle = {Proceedings of the Twentieth International Conference on International Conference on Machine Learning},
pages = {928–935},
numpages = {8},
location = {Washington, DC, USA},
series = {ICML'03}
}


\newpage
\appendix

\tableofcontents

\newpage

\section{Proofs for Section~\ref{sec:full-info}: Proof of Theorem~\ref{thm:main}}
\label{sec:analysis}

In this section, we provide a complete proof of Theorem~\ref{thm:main}. The proof strategy consists of showing that the OGD iterates mapped to $\Y$ via $q$ and the OMD iterates are close to each other in the hidden space~$\Y$ after a single update step starting from the same initial point. Then we can view the OGD update as a perturbed version of OMD, and combine it with the fact that the OMD algorithm can tolerate bounded noise per trial. 

\subsection{Proof of Lemma~\ref{prop:dist-minimizers-eta2-main}}
\label{appx:proof-main-lemma}

We begin with the following key lemma showing that the updates $y_{t+1}$ and $z_{t+1} = q(x_{t+1})$ created by OMD and OGD respectively, are close to each other when starting from the same initial point $y_t= z_t = q(x_t)$. We show that the two iterates are minimizers of approximately the same strongly-convex objective, hence the minimizers must be close. 
In preparation of the bandit setting, we prove a general result that holds pathwise when we use any bounded vector instead of the exact gradients of both the original loss function $\ell_t$ and the hidden function $h_t.$ 

\begin{tcolorbox}[colframe=white!, top=2pt,left=2pt,right=2pt,bottom=2pt] 
\begin{proposition}\label{prop:dist-minimizers-eta2}
Suppose Assumptions~\ref{as:reparam}-\ref{as:boundedness-grad-bregman} hold and assume $y_t= q(x_t)$.
Consider the following OMD and OGD-like update rules: 
\begin{align}
y_{t+1}&=\argmin_{y\in \Y} \tilde{g}_t^{\top}(y-y_t) +\frac{1}{\eta} D_R(y\|y_t)\,, \label{eq:omd}\\
x_{t+1}&=\argmin_{x\in \X} g_t^{\top}(x-x_t) +\frac{1}{2\eta} \|x-x_t\|_2^2 \label{eq:ogd}\,, 
\end{align}
where $\tilde{g}_t = J_q(x_t)^{-\top} g_t$ and $g_t \in \mathbb{R}^d$ is any vector\footnote{$\tilde{g}_t$ is used in the analysis but never actually computed, in contrast to $g_t$ which is computed and used in BOGD. Note that we suppose that we do not have access to the parameterization $q$.}. If in addition there exists $\hat{G}_F > 0$ such that $\|g_t\| \leq \hat{G}_F$ for any $t$, then,  
\begin{equation}
\label{eq:approx-ineq}
    \|y_{t+1}-q(x_{t+1})\|_2 \leq G^5 (5\hat{G}_F^2 + \hat{G}_F^3 \eta) \eta^2\,.
\end{equation} 
\end{proposition}
\end{tcolorbox}

Before providing the proof, we state an immediate corollary of this result in the exact gradient information setting of Theorem~\ref{thm:main}. 

\begin{tcolorbox}[colframe=white!, top=2pt,left=2pt,right=2pt,bottom=2pt] 
\begin{corollary}[Proposition~\ref{prop:dist-minimizers-eta2} in the exact gradient setting]
\label{cor:approx-lemma-det}
In Proposition~\ref{prop:dist-minimizers-eta2}, set $g_t = \nabla \ell_t(x_t)$ and $\tilde{g}_t = \nabla h_t(z_t)$ in \eqref{eq:omd}-\eqref{eq:ogd}, then the approximation inequality \eqref{eq:approx-ineq} holds with $\hat{G}_F = G_F$ where $G_F$ is the uniform loss gradient ($\nabla \ell_t$) bound in Assumption~\ref{as:boundedness-grad-bregman}, i.e., 
\begin{equation*}
\|y_{t+1}-q(x_{t+1})\|_2 \leq C_\eta := G^5(5G_F^2+G_F^3\eta)\eta^2.
\end{equation*}
\end{corollary}
\end{tcolorbox}

\begin{proof}[Proof of Proposition~\ref{prop:dist-minimizers-eta2}] 
Note first that $\|\tilde{g}_t\| \leq \tilde{G}_F := G \hat{G}_F\,.$ 
Both objectives in \eqref{eq:omd}-\eqref{eq:ogd} can be written as the sum of a linear function and a strongly convex function. 
By 1-strong convexity of $R$ and uniform boundedness of $g_t$ by $\hat{G}_F$, we have for any $y \in \Y$, 
\begin{equation*}
\tilde{g}_t^{\top}(y-y_t) +\frac{1}{\eta} D_R(y\|y_t) \geq - \tilde{G}_F \|y-y_t\| + \frac{1}{2\eta} \|y-y_t\|^2\,.
\end{equation*}
Therefore, if $\|y-y_t\|_2>2 \eta \tilde{G}_F$, then the objective takes a positive value which cannot be minimal since $y=y_t$ gives zero objective value in \eqref{eq:omd}. Define the set:  
\begin{equation*}
\Y_{r,y_t}=\Y \cap \{y \in \reals^d: \|y-y_t\|_2 \le r\}\,.
\end{equation*}
Then it follows that the OMD iterate in \eqref{eq:omd} satisfies: 
\begin{equation}
\label{eq:y_t+1inY_t}
y_{t+1} \in \Y_t := \Y_{2\eta \tilde{G}_F,y_t}\,. 
\end{equation}
As we want to localize both iterates \eqref{eq:omd}-\eqref{eq:ogd} in a set of size controlled by $\eta$, we now show that we also have $q(x_{t+1}) \in \Y_t$ when $y_t = q(x_t)$. Observe for this that: 
\begin{equation*}
\|q(x_{t+1}) - y_t\| = \|q(x_{t+1}) - q(x_t)\| 
\leq G \|x_{t+1} - x_t \|
\leq \eta G \|g_t\| 
\leq \eta G \hat{G}_F 
\leq 2 \eta G \hat{G}_F 
= 2 \eta \tilde{G}_F\,. 
\end{equation*}
We now write both update rules under a similar form. 
Define the function $\Phi_t: \Y \to \mathbb{R}$ by: 
\begin{equation}
\label{eq:Phi_t}
\Phi_t(y) := \tilde{g}_t^{\top}(y - y_t) + \frac{1}{2\eta} \|y-y_t\|_{\nabla^2 R(y_t)}^2\,.
\end{equation}
The function $\Phi_t$ is $1/\eta$-strongly convex because
$\nabla^2\Phi_t(y)=
\frac1\eta\nabla^2R(y_t)
\succeq
\frac1\eta I$ as $R$ is supposed to be $1$-strongly convex.  

\begin{tcolorbox}[colframe=white!, top=2pt,left=2pt,right=2pt,bottom=2pt] 
\begin{lemma}
\label{lem:omd-rewritten}
For any $t \geq 1,$ the iterates of \eqref{eq:omd-main}-\eqref{eq:omd} can be rewritten as follows:
\begin{equation*}
y_{t+1} =\argmin_{y\in \Y_t} \left\{ \Phi_t(y) + \varepsilon_t^{\mathrm{OMD}}(y) \right\}\,, \quad \varepsilon_t^{\mathrm{OMD}}(y) := \frac{1}{\eta} \varepsilon_t(y)\,,
\end{equation*}
where $\varepsilon_t(y):= D_R(y\|y_t) - \frac12 \|y-y_t\|_{\nabla^2 R(y_t)}^2\,.$ Moreover, for $y \in \Y_t$, $|\varepsilon_t(y)| \leq \frac{4}{3} G \tilde{G}_F^3 \eta^3\,.$
\end{lemma}
\end{tcolorbox}

\begin{proof}
We have proved in \eqref{eq:y_t+1inY_t} that $y_{t+1} \in \Y_t$.  Hence, we can rewrite \eqref{eq:omd} by replacing $\Y$ by $\Y_t$: 
\begin{equation}
\label{eq:intermediate-x-rewritten}
y_{t+1} =\argmin_{y\in \Y_t} \tilde{g}_t^{\top}(y-y_t) +\frac{1}{\eta} D_R(y\|y_t). 
\end{equation}

Let $\Delta_t :=y-y_t$. By Taylor's theorem with integral remainder, we have: 
\begin{align}
\label{eq:bregman-approx}
D_R(y\|y_t)
&=
R(y)-R(y_t)-\nabla R(y_t)^\top (y-y_t) \nonumber\\
&=
\frac12 \Delta_t^\top \nabla^2 R(y_t)\Delta_t
+
\varepsilon_t(y) \nonumber\\
&=
\frac12 \|\Delta_t\|_{\nabla^2 R(y_t)}^2
+
\varepsilon_t(y),
\end{align}
where $\varepsilon_t(y)$ is the integral remainder defined as follows: 
\begin{equation*}
\varepsilon_t(y)
:= 
\frac12
\int_0^1
(1-s)^2
\nabla^3 R(y_t+s \Delta_t)[\Delta_t,\Delta_t,\Delta_t]\,ds\,,
\end{equation*}
where $\nabla^3 R$ is the third-order derivative tensor of $R$. 
Since this third derivative is uniformly bounded by $G$ over $z \in \Y$ by Assumption~\ref{as:reg-R-reparam-q-smoothness}, we have:
\begin{equation*}
|\varepsilon_t(y)|
\le
\frac{G}{6}\|\Delta_t\|_2^3.
\end{equation*}
In particular, for \(y\in \mathcal Y_t\), i.e.
\(\|y-y_t\|_2\le 2 \eta \tilde{G}_F\), we obtain the following estimate: 
\begin{equation*}
|\varepsilon_t(y)| \leq \frac{G}{6} (2 \eta \tilde{G}_F)^3 = \frac{4}{3} G \tilde{G}_F^3 \eta^3\,.
\end{equation*}
Combining the Bregman approximation of \eqref{eq:bregman-approx} with \eqref{eq:intermediate-x-rewritten}, we obtain the desired identity using the definition of the strongly-convex function~$\Phi_t$ introduced in \eqref{eq:Phi_t}. 
\end{proof}

In the following lemma, we rewrite the OGD iterates mapped to $\Y$ via $q$ under a similar form to the update rule of OMD obtained in Lemma~\ref{lem:omd-rewritten}. 

\begin{tcolorbox}[colframe=white!, top=2pt,left=2pt,right=2pt,bottom=2pt] 
\begin{lemma}
\label{lem:ogd-rewritten}
For the update~\eqref{eq:ogd} with a bounded vector $g_t$,
suppose $y_t=q(x_t)$. Then, for $t \geq 1,$
\begin{align}
\label{eq:q(u_t+1)}
z_{t+1} = q(x_{t+1}) &=\argmin_{y\in \Y_t} \left\{ \Phi_t(y) + \varepsilon_t^{\mathrm{OGD}}(y) \right\}\,,\\
\varepsilon_t^{\mathrm{OGD}}(y) &:= \tilde\varepsilon_t(y)
+ \frac{1}{\eta} \left(\frac12\|\varepsilon_t^q(y)\|_2^2 + \big\langle J_q(x_t)^{-1}(y-y_t),\,\varepsilon_t^q(y)\big\rangle
\right)\,,
\end{align}
where $\tilde\varepsilon_t(y) := g_t^\top\varepsilon_t^q(y)$ and $\varepsilon_t^q(y) := q^{-1}(y)-q^{-1}(y_t)- J_q^{-1}(x_t)(y-y_t)\,.$ Moreover, for every $y\in\Y_t$, we have $\|\varepsilon_t^q(y)\|_2 \leq 2 G \tilde{G}_F^2 \eta^2$ and 
$|\tilde\varepsilon_t(y)| \leq 2 \tilde{G}_F^3 \eta^2$
\end{lemma}
\end{tcolorbox}

\begin{proof}
Recall that the update~\eqref{eq:ogd} in the original space $\X$
can be written as follows:
\begin{equation}
\label{eq:update-rule-OGD}
x_{t+1}
=
\argmin_{x\in \X}
\left\{
g_t^\top(x-x_t)+\frac{1}{2\eta}\|x-x_t\|_2^2
\right\}.
\end{equation}

Recall from \eqref{eq:Phi_t} that for any $y \in \Y$: 
\begin{equation*}
\Phi_t(y) := \tilde{g}_t^{\top}(y - y_t) + \frac{1}{2\eta} \|y-y_t\|_{\nabla^2 R(y_t)}^2\,.
\end{equation*}
To prove \eqref{eq:q(u_t+1)}, we relate the linear term $g_t^\top (x-x_t)$ to $\tilde{g}_t^\top(y-y_t)$ and the quadratic term $\frac{1}{2\eta}\|x-x_t\|_2^2$ to $\frac{1}{2\eta} \|y-y_t\|_{\nabla^2 R(y_t)}^2$ respectively in \textbf{(i)} and \textbf{(ii)} below. 

\noindent\textbf{(i) Relating $g_t^\top (x-x_t)$ to $ \tilde{g}_t^\top(y-y_t).$}
Recalling that $g_t = J_q(x_t)^{\top} \tilde{g}_t$, we have:  
\begin{align}
\label{eq:inner-prod-tilde-f}
g_t^\top (x-x_t)
&= \tilde{g}_t^\top J_q(x_t)(x-x_t) \nonumber\\
&= \tilde{g}_t^\top J_q(x_t)(q^{-1}(y) - q^{-1}(y_t))\,,
\end{align}
where the last equality follows from using the identities $y = q(x)$ and $y_t = q(x_t)\,.$ 

Under Assumption~\ref{as:reg-R-reparam-q-smoothness}, applying Taylor's theorem to $q^{-1}$ around $y_t=q(x_t)$ gives, for $y=q(x)$,
\begin{equation}
\label{eq:taylor-q-1}
q^{-1}(y)-q^{-1}(y_t)
=
J_{q^{-1}}(y_t)(y-y_t)
+
\varepsilon_t^q(y)\,,
\end{equation}
where the remainder is given by: 
\begin{equation}
\label{eq:remainder-eps-t-q}
\varepsilon_t^q(y) 
= \int_0^1 (1-s)\, (y-y_t)^{\top}
\nabla^2 q^{-1}(y_t+s(y-y_t))
(y-y_t)\,ds.
\end{equation}
Since \(J_{q^{-1}}(y_t)=J_q(x_t)^{-1}\), it follows that: 
\begin{equation}
\label{eq:taylor-q-2}
q^{-1}(y)-q^{-1}(y_t) = J_q(x_t)^{-1}(y-y_t) + \varepsilon_t^q(y).
\end{equation}
Since the second derivative of \(q^{-1}\) is bounded by \(G\), we have:
\begin{equation*}
\|\varepsilon_t^q(y)\|_2
\le
\frac{G}{2}\|y-y_t\|_2^2.
\end{equation*}
In particular, for $y\in \Y_t = \mathcal Y_{2\eta \tilde{G}_F,y_t}$, i.e. $y \in \Y$ s.t.
$\|y-y_t\|_2\le 2\eta \tilde{G}_F$, we obtain: 
\begin{equation}
\label{eq:estimate-vareps-tqy}
\|\varepsilon_t^q(y)\|_2 \leq \frac{G}{2} (2\eta \tilde{G}_F)^2 = 2 G \tilde{G}_F^2 \eta^2\,.
\end{equation}

Using \eqref{eq:taylor-q-2} in \eqref{eq:inner-prod-tilde-f} yields:  
\begin{align}
\label{eq:term-inner-prod-rewriting}
g_t^\top (x-x_t)
&= \tilde{g}_t^\top J_q(x_t)(q^{-1}(y) - q^{-1}(y_t)) \nonumber\\
&= \tilde{g}_t^\top(y-y_t) + \tilde{g}_t^\top J_q(x_t)\,\varepsilon_t^q(y) \nonumber\\
&= \tilde{g}_t^\top(y-y_t) + \tilde\varepsilon_t(y),
\end{align}
where 
$\tilde\varepsilon_t(y) :=
\tilde{g}_t^\top J_q(x_t)\,\varepsilon_t^q(y)
= g_t^{\top} \varepsilon_t^q(y)$. Therefore, using \eqref{eq:estimate-vareps-tqy}, we have for any $y \in \Y_t,$ 
\begin{equation*}
|\tilde\varepsilon_t(y)| \leq \|g_t\| \cdot \|\varepsilon_t^q(y)\| \leq G (2\hat{G}_F \tilde{G}_F^2 \eta^2) = 2 \tilde{G}_F^3 \eta^2\,,
\end{equation*}
where the last equality stems from recalling that $\tilde{G}_F = G \hat{G}_F\,.$ 

\noindent\textbf{(ii) Relating $\frac{1}{2\eta}\|x-x_t\|_2^2$ to $\frac{1}{2\eta} \|y-y_t\|_{\nabla^2 R(y_t)}^2.$} For any $x \in \X$, $y = q(x),$ given that $y_t = q(x_t)$,  
\begin{align}
\label{eq:u-u_t-square}
\frac{1}{2\eta}\|x-x_t\|_2^2
&= \frac{1}{2\eta}\|q^{-1}(y)-q^{-1}(y_t)\|_2^2 \nonumber\\
&= \frac{1}{2\eta}\|J_q(x_t)^{-1}(y-y_t) + \varepsilon_t^q(y)\|_2^2 \nonumber\\
&= \frac{1}{2\eta}\|J_q(x_t)^{-1}(y-y_t)\|_2^2
+ \frac{1}{\eta} \left(\frac12\|\varepsilon_t^q(y)\|_2^2
+\big\langle J_q(x_t)^{-1}(y-y_t),\,\varepsilon_t^q(y)\big\rangle\right).
\end{align}
Using Assumption~\ref{as:reparam} relating $J_q$ and $\nabla^2R$ and the initial coupling $y_t = z_t = q(x_t)$, we have: 
\begin{equation*}
\big[\nabla^2R(y_t)\big]^{-1}=J_q(x_t)J_q(x_t)^\top,
\qquad
\nabla^2R(y_t)=[J_q(x_t)^{\top}]^{-1}J_q(x_t)^{-1}\,.
\end{equation*}
Therefore we can write: 
\begin{align}
\label{eq:Jq-1(x-x_t)}
\|J_q(x_t)^{-1}(y-y_t)\|_2^2
&= (y-y_t)^{\top} [J_q(x_t)^{-1}]^{\top} J_q(x_t)^{-1}(y-y_t) \nonumber\\
&= (y-y_t)^{\top} \nabla^2R(y_t)(y-y_t) \nonumber\\
&= \|y-y_t\|^2_{\nabla^2R(y_t)}.
\end{align}
Combining \eqref{eq:u-u_t-square} and \eqref{eq:Jq-1(x-x_t)} yields: 
\begin{equation}
\label{eq:term-reg-rewriting}
\frac{1}{2\eta}\|x-x_t\|_2^2 = \frac{1}{2\eta} \|y-y_t\|^2_{\nabla^2R(y_t)} 
+ \frac{1}{\eta} \left(\frac12\|\varepsilon_t^q(y)\|_2^2
+\big\langle J_q(x_t)^{-1}(y-y_t),\,\varepsilon_t^q(y)\big\rangle\right)\,.
\end{equation}

Given the OGD update rule \eqref{eq:update-rule-OGD}, the change of variable $y = q(x)$ and \eqref{eq:term-inner-prod-rewriting}, \eqref{eq:term-reg-rewriting} proven in \textbf{(i)} and \textbf{(ii)} respectively, we have shown that 
\begin{equation*}
q(x_{t+1})
=
\argmin_{y\in \Y_t}
\left\{
\Phi_t(y)+\varepsilon_t^{\mathrm{OGD}}(y)
\right\},
\end{equation*}
where the error term $\varepsilon_t^{\mathrm{OGD}}(y)$ is defined as follows: 
\begin{equation}
\label{eq:eps-OGD}
\varepsilon_t^{\mathrm{OGD}}(y)
:=
\tilde\varepsilon_t(y)
+
\frac{1}{\eta}
\left(
\frac12\|\varepsilon_t^q(y)\|_2^2
+
\big\langle J_q(x_t)^{-1}(y-y_t),\,\varepsilon_t^q(y)\big\rangle
\right).
\end{equation}
\end{proof}

So far, we have shown in Lemmas~\ref{lem:omd-rewritten} and~\ref{lem:ogd-rewritten} respectively the following similar update rule forms: 
\begin{align}
y_{t+1}
&=
\argmin_{y\in \Y_t}\{\Phi_t(y)+\varepsilon_t^{\mathrm{OMD}}(y)\},
\label{eq:min1}
\\
z_{t+1} = q(x_{t+1})
&=
\argmin_{y\in \Y_t}\{\Phi_t(y)+\varepsilon_t^{\mathrm{OGD}}(y)\}.
\label{eq:min2}
\end{align}
In the next lemma, we use first-order conditions together with strong convexity of the function~$\Phi_t$ to show that the distance between the minimizers can be upper bounded by $\eta$ times the distance between the gradients of the errors in each one of the updates (for OMD and OGD). Note that our proof technique here is different from the strategy adopted in \cite{ghai-lu-hazan22} (Lemma 4) as we do not control the distance between objectives and translate it to the distance between minimizers using strong convexity. 

\begin{tcolorbox}[colframe=white!, top=2pt,left=2pt,right=2pt,bottom=2pt] 
\begin{lemma} 
\label{lem:dist-minimizers-eta-times-norm}
For any $t \geq 1$, if $y_t = q(x_t)$, then 
\begin{equation*}
\|y_{t+1} - q(x_{t+1})\| \le \eta\,
\left\|\nabla\varepsilon_t^{\mathrm{OGD}}(q(x_{t+1})) - \nabla\varepsilon_t^{\mathrm{OMD}}(y_{t+1})\right\|\,.
\end{equation*}
\end{lemma}
\end{tcolorbox}

To prove this result, we will need the following technical lemma stating monotonicity of normal cones of nonempty closed convex sets, which is a standard result in convex analysis. 

\begin{lemma}[Monotonicity of the normal cone]
\label{lem:normal-cone-monotonicity}
Let \(K\subset \mathbb R^d\) be a nonempty closed convex set. For any
\(x,y\in K\), let
\[
N_K(x):=\{v\in\mathbb R^d:\langle v,z-x\rangle\le 0
\,, \quad \forall z\in K\}
\]
denote the normal cone of \(K\) at \(x\). Then, for any
\(v_x\in N_K(x)\) and \(v_y\in N_K(y)\),
\[
\langle v_x-v_y,x-y\rangle \ge 0.
\]
\end{lemma}

\begin{proof}
Since \(v_x\in N_K(x)\) and \(y\in K\), the definition of the normal cone gives $\langle v_x,y-x\rangle\le 0,$
or equivalently, $\langle v_x,x-y\rangle\ge 0.$
Similarly, since \(v_y\in N_K(y)\) and \(x\in K\), we have 
$\langle v_y,x-y\rangle\le 0.$
Therefore, it follows from the above inequalities that: 
\begin{equation*}
\langle v_x-v_y,x-y\rangle
=
\langle v_x,x-y\rangle
-
\langle v_y,x-y\rangle
\ge 0\,.
\end{equation*}
\end{proof}

\begin{proof}[Proof of Lemma~\ref{lem:dist-minimizers-eta-times-norm}]
Let $\delta_{t+1}:=y_{t+1}-q(x_{t+1})$ as a shorthand notation.  
From first-order optimality conditions for \eqref{eq:min1} and \eqref{eq:min2},
there exist $v_{t+1}^{\mathrm{OMD}} \in N_{\Y_t}(y_{t+1}), v_{t+1}^{\mathrm{OGD}} \in N_{\Y_t}(q(x_{t+1}))$ s.t. 
\begin{align*}
\nabla \Phi_t(y_{t+1})+\nabla \varepsilon_t^{\mathrm{OMD}}(y_{t+1})+v_{t+1}^{\mathrm{OMD}} &= 0\,,\\
\nabla \Phi_t(q(x_{t+1}))+\nabla \varepsilon_t^{\mathrm{OGD}}(q(x_{t+1}))+v_{t+1}^{\mathrm{OGD}} &= 0\,.
\end{align*}
Subtracting both identities and taking inner product with $\delta_{t+1}$ gives
\begin{multline}
\big\langle \nabla\Phi_t(y_{t+1})-\nabla\Phi_t(q(x_{t+1})),\,\delta_{t+1}\big\rangle
+ \big\langle \nabla\varepsilon_t^{\mathrm{OMD}}(y_{t+1})
- \nabla\varepsilon_t^{\mathrm{OGD}}(q(x_{t+1})),\,\delta_{t+1}\big\rangle \\
+ \big\langle v_{t+1}^{\mathrm{OMD}}-v_{t+1}^{\mathrm{OGD}}, \,\delta_{t+1} \big\rangle
=0.
\label{eq:master}
\end{multline}
By $\frac{1}{\eta}$-strong convexity of $\Phi_t$, we have: 
\begin{equation}
\label{eq:sc-phi_t-ineq}
\big\langle \nabla\Phi_t(y_{t+1})-\nabla\Phi_t(q(x_{t+1})), \,\delta_{t+1}\big\rangle
\ge \frac1\eta \|\delta_{t+1}\|^2.
\end{equation}
By monotonicity of normal cones of convex sets (Lemma~\ref{lem:normal-cone-monotonicity} above),
we can also control the term involving the normal cone directions as follows: 
\begin{equation}
\label{eq:normal-cone-monotonicity-ineq}
\big\langle v_{t+1}^{\mathrm{OMD}}-v_{t+1}^{\mathrm{OGD}}, \,\delta_{t+1} \big\rangle \ge 0.
\end{equation}
Plugging \eqref{eq:sc-phi_t-ineq} and \eqref{eq:normal-cone-monotonicity-ineq} into \eqref{eq:master}, we obtain: 
\begin{equation*}
\frac1\eta\|\delta_{t+1}\|^2
\le 
\left\| \nabla\varepsilon_t^{\mathrm{OGD}}(q(x_{t+1})) - \nabla\varepsilon_t^{\mathrm{OMD}}(y_{t+1}) \right\|
\cdot \|\delta_{t+1}\|.
\end{equation*}
Therefore, it follows that the distance~$\delta_{t+1}$ between the two minimizers is upper bounded as follows: 
\[
\|\delta_{t+1}\|
\le
\eta\,
\left\|
\nabla\varepsilon_t^{\mathrm{OGD}}(q(x_{t+1}))
-
\nabla\varepsilon_t^{\mathrm{OMD}}(y_{t+1})
\right\|.
\]
\end{proof}

It remains to bound the last norm in the right-hand side of Lemma~\ref{lem:dist-minimizers-eta-times-norm} and show it is also of order $\eta$. We show that each one of the norms of the gradients $\|\nabla\varepsilon_t^{\mathrm{OMD}}(y_{t+1})\|$ and $\|\nabla\varepsilon_t^{\mathrm{OGD}}(q(x_{t+1}))\|$ are of order $\eta$ in the following two lemmas. 

\begin{tcolorbox}[colframe=white!, top=2pt,left=2pt,right=2pt,bottom=2pt] 
\begin{lemma}
\label{lem:norm-grad-eps-OMD}
$\|\nabla \varepsilon_t^{\mathrm{OMD}}(y_{t+1})\| \leq 2 G \tilde{G}_F^2 \eta\,.$
\end{lemma}
\end{tcolorbox}

\begin{proof}
Recall that $\varepsilon_t^{\mathrm{OMD}}(y)=\frac1\eta \varepsilon_t(y),$
where: 
\begin{equation*}
\varepsilon_t(y)
=
D_R(y\|y_t)-\frac12(y-y_t)^\top \nabla^2R(y_t)(y-y_t)\,.
\end{equation*}
Differentiating the above error function w.r.t. $y$ yields: 
\begin{equation*}
\nabla \varepsilon_t(y)
=
\nabla R(y)-\nabla R(y_t)-\nabla^2R(y_t)(y-y_t).
\end{equation*}
Using third-order smoothness of $R$ (bounded third derivative) as supposed in Assumption~\ref{as:reg-R-reparam-q-smoothness},
\begin{equation}
\label{eq:bound-grad-eps-t}
\|\nabla \varepsilon_t(y_{t+1})\|
\le
\frac{G}{2}\|y_{t+1}-y_t\|^2.
\end{equation}
Since $y_{t+1}\in \Y_t$, we have $\|y_{t+1}-y_t\| \leq 2\eta \tilde{G}_F$ and it follows from \eqref{eq:bound-grad-eps-t} that: 
\begin{equation*}
\|\nabla \varepsilon_t(y_{t+1})\| \leq \frac{G}{2} (2\eta \tilde{G}_F)^2 = 2 G \tilde{G}_F^2 \eta^2\,.
\end{equation*}
Since $\varepsilon_t^{\mathrm{OMD}}(y)=\frac1\eta \varepsilon_t(y),$ we conclude that: 
\begin{equation*}
\|\nabla \varepsilon_t^{\mathrm{OMD}}(y_{t+1})\|
= \frac 1 \eta \|\nabla \varepsilon_t(y_{t+1})\| 
\leq 2 G \tilde{G}_F^2 \eta\,.
\end{equation*}
\end{proof}

\begin{tcolorbox}[colframe=white!, top=2pt,left=2pt,right=2pt,bottom=2pt] 
\begin{lemma}
\label{lem:norm-grad-eps-OGD}
$\|\nabla\varepsilon_t^{\mathrm{OGD}}(q(x_{t+1}))\| \leq G^5 (3\hat{G}_F^2 + \hat{G}_F^3 \eta) \eta.$
\end{lemma}
\end{tcolorbox}

\begin{proof}
Recall from \eqref{eq:eps-OGD} that: 
\[
\varepsilon_t^{\mathrm{OGD}}(y)
=
\tilde\varepsilon_t(y)
+
\frac1\eta
\left(
\frac12\|\varepsilon_t^q(y)\|_2^2
+
\big\langle J_q(x_t)^{-1}(y-y_t),\,\varepsilon_t^q(y)\big\rangle
\right).
\]
Differentiating this error term w.r.t. $y$ gives: 
\begin{equation}
\label{eq:grad-eps-OGD}
\nabla\varepsilon_t^{\mathrm{OGD}}(y)
=
\nabla \tilde\varepsilon_t(y)
+
\frac1\eta J_{\varepsilon_t^q}(y)^\top \varepsilon_t^q(y)
+
\frac1\eta\Big(
[J_q(x_t)^{-1}]^\top \varepsilon_t^q(y)
+
J_{\varepsilon_t^q}(y)^\top J_q(x_t)^{-1}(y-y_t)
\Big).
\end{equation}
Since $\tilde\varepsilon_t(y) := g_t^\top \varepsilon_t^q(y)$, we have: 
\begin{equation*}
\nabla \tilde\varepsilon_t(y)
=
J_{\varepsilon_t^q}(y)^\top g_t\,, \quad \|\nabla \tilde\varepsilon_t(y)\| \leq \hat{G}_F \|J_{\varepsilon_t^q}(y)\|\,.
\end{equation*}

We now prove the following estimates: 
\begin{enumerate}[label=(\roman*)]
\item \label{eq:i} $\|q(x_{t+1})-y_t\| \leq  G \hat{G}_F \eta\,,$
\item \label{eq:ii} $\|J_{\varepsilon_t^q}(q(x_{t+1}))\| \leq  G^2 \hat{G}_F \eta,$
\item \label{eq:iii} $\|\varepsilon_t^q(q(x_{t+1}))\| \leq \frac{G^3 \hat{G}_F^2}{2} \eta^2\,.$ 
\end{enumerate}

\noindent\textbf{Proof of \ref{eq:i}.}
Since $y_t=q(x_t)$, $x_t\in\X$, and
$x_{t+1}=\Pi_\X(x_t-\eta g_t)$,
\begin{align*}
\|q(x_{t+1})-y_t\|_2
&=
\|q(x_{t+1})-q(x_t)\|_2\\
&\le
G\|x_{t+1}-x_t\|_2\\
&=
G\|\Pi_\X(x_t-\eta g_t)-\Pi_\X(x_t)\|_2\\
&\le
\eta G\|g_t\|_2\\
&\le
\eta G\hat G_F
=
\eta\tilde G_F.
\end{align*}
Here we used Lipschitzness of $q$, nonexpansiveness of Euclidean
projection, and the assumed bound $\|g_t\|_2\le\hat G_F$.

\noindent\textbf{Proof of \ref{eq:ii}.} Recall from \eqref{eq:taylor-q-1} that for any $y=q(x), x \in \X\,,$ 
\begin{equation}
\label{eq:eps-t-q-taylor}
\varepsilon_t^q(y) = q^{-1}(y)-q^{-1}(y_t) - J_{q^{-1}}(y_t)(y-y_t)\,.
\end{equation}
Differentiating w.r.t. $y$ yields: 
\begin{equation*}
J_{\varepsilon_t^q}(q(x_{t+1})) = J_{q^{-1}}(q(x_{t+1})) - J_{q^{-1}}(y_t)\,.
\end{equation*}
Using Assumption~\ref{as:reg-R-reparam-q-smoothness}, we obtain: 
\begin{equation*}
\|J_{\varepsilon_t^q}(q(x_{t+1}))\| = \|J_{q^{-1}}(q(x_{t+1})) - J_{q^{-1}}(y_t)\| \leq G \|q(x_{t+1}) - y_t\| \leq \eta G^2 \hat{G}_F\,,
\end{equation*}
where the last step uses the first estimate \ref{eq:i} proved above.

\noindent\textbf{Proof of \ref{eq:iii}.} Recalling \eqref{eq:eps-t-q-taylor} and using boundedness of the second-order derivative of $q^{-1}$ (Assumption~\ref{as:reg-R-reparam-q-smoothness}), we immediately obtain: 
\begin{equation*}
\|\varepsilon_t^q(q(x_{t+1}))\| \leq \frac{G}{2} \cdot \|q(x_{t+1}) - y_t\|^2 \leq \frac{G^3 \hat{G}_F^2}{2} \eta^2\,,
\end{equation*}
where the last estimate follows from using the first proven estimate \ref{eq:i}: $\|q(x_{t+1})-y_t\|\leq  G \hat{G}_F \eta\,.$

We conclude the proof of Lemma~\ref{lem:norm-grad-eps-OGD} by using all the estimates \ref{eq:i} to \ref{eq:iii} to bound the norm of the gradient~$\nabla\varepsilon_t^{\mathrm{OGD}}(x)$ in \eqref{eq:grad-eps-OGD} as follows:
\begin{align}
\label{eq:grad-eps-ogd}
\|\nabla\varepsilon_t^{\mathrm{OGD}}(q(x_{t+1}))\| 
&\leq \hat{G}_F \|J_{\varepsilon_t^q}(q(x_{t+1}))\| 
+ \frac{1}{\eta} \|J_{\varepsilon_t^q}(q(x_{t+1}))\| \cdot \|\varepsilon_t^q(q(x_{t+1}))\| \nonumber\\
&+ \frac{1}{\eta} G \|\varepsilon_t^q(q(x_{t+1}))\| + \frac{1}{\eta} G \|J_{\varepsilon_t^q}(q(x_{t+1}))\| \cdot \|q(x_{t+1})-y_t\| \nonumber\\
&\leq G^2  \hat{G}_F^2 \eta + \frac 12 G^5 \hat{G}_F^3 \eta^2 + \frac 12 G^4 \hat{G}_F^2 \eta + G^4 \hat{G}_F^2 \eta\,\nonumber\\
&\leq G^5 (3\hat{G}_F^2 + \hat{G}_F^3 \eta) \eta\,,
\end{align}
where the last inequality follows from using the assumption $G \geq 1$ which is without loss of generality (otherwise replace the constant by 1). 
\end{proof}

\noindent\textbf{End of Proof of Proposition~\ref{prop:dist-minimizers-eta2}.} 
Using Lemma~\ref{lem:dist-minimizers-eta-times-norm}, we have: 
\begin{equation*}
\|y_{t+1} - q(x_{t+1})\| \le \eta\,
\left\|\nabla\varepsilon_t^{\mathrm{OGD}}(q(x_{t+1})) - \nabla\varepsilon_t^{\mathrm{OMD}}(y_{t+1})\right\|\,.
\end{equation*}
Then it follows from using Lemma~\ref{lem:norm-grad-eps-OMD} and Lemma~\ref{lem:norm-grad-eps-OGD} that: 
\begin{align}
\|y_{t+1}-q(x_{t+1})\| = 
\|\delta_{t+1}\|
&\le
\eta\cdot
\left(
\|\nabla\varepsilon_t^{\mathrm{OGD}}(q(x_{t+1}))\|
+
\|\nabla\varepsilon_t^{\mathrm{OMD}}(y_{t+1})\|
\right)\nonumber\\
&\leq (G^5 (3\hat{G}_F^2 + \hat{G}_F^3 \eta) + 2 G \tilde{G}_F^2) \eta^2 \nonumber\\
&= (G^5 (3\hat{G}_F^2 + \hat{G}_F^3 \eta) + 2 G^3 \hat{G}_F^2)\eta^2\nonumber\\
&\leq G^5 (5\hat{G}_F^2 + \hat{G}_F^3 \eta) \eta^2\,,
\end{align}
which concludes the proof. 
\end{proof}

\subsection{Proof of Theorem~\ref{thm:main} using Lemma~\ref{prop:dist-minimizers-eta2-main}}

The following lemma controls the regret of an approximate OMD algorithm.

\begin{tcolorbox}[colframe=white!,
top=2pt,left=2pt,right=2pt,bottom=2pt]
\begin{lemma}[Perturbed OMD regret bound;
cf.~Lemma~6 in \cite{ghai-lu-hazan22}]
\label{lem2}
Let Assumptions~\ref{as:reparam}--\ref{as:boundedness-grad-bregman}
hold and let $\eta>0$.
Suppose $(z_t)_{t=1}^{T+1}\subseteq\Y$ satisfies
\[
y_{t+1}
=
\argmin_{y\in\Y}
\left\{
\langle\nabla h_t(z_t),y-z_t\rangle
+
\frac1\eta D_R(y\|z_t)
\right\},
\qquad
z_{t+1}=y_{t+1}+r_{t+1},
\]
where $\|r_{t+1}\|_2\le C_\eta$ for all $t$.
Then
\[
\sum_{t=1}^T h_t(z_t)
-
\min_{z\in\Y}\sum_{t=1}^T h_t(z)
\le
\frac{D_1}{\eta}
+
\frac{\eta G^2G_F^2T}{2}
+
\frac{GC_\eta T}{\eta}.
\]
\end{lemma}
\end{tcolorbox}

\begin{proof}
Fix any comparator $z\in\Y$.
The optimality condition for the ghost OMD step, the Bregman
three-point identity, and $1$-strong convexity of $R$ give: 
\[
\langle\nabla h_t(z_t),z_t-z\rangle
\le
\frac{D_R(z\|z_t)-D_R(z\|y_{t+1})}{\eta}
+
\frac{\eta}{2}\|\nabla h_t(z_t)\|_2^2.
\]
By the chain rule and Assumptions~
\ref{as:reg-R-reparam-q-smoothness}--\ref{as:boundedness-grad-bregman},
\[
\|\nabla h_t(z_t)\|_2
=
\|J_q(q^{-1}(z_t))^{-\top}
  \nabla\ell_t(q^{-1}(z_t))\|_2
\le GG_F.
\]
Moreover, both $z_{t+1}$ and $y_{t+1}$ belong to $\Y$, so
Assumption~\ref{as:reg-R-reparam-q-smoothness}(iii) implies
\[
D_R(z\|z_{t+1})-D_R(z\|y_{t+1})
\le
G\|z_{t+1}-y_{t+1}\|_2
\le GC_\eta.
\]
Using convexity of $h_t$ and inserting $z_{t+1}$ into the
Bregman difference, we obtain
\[
h_t(z_t)-h_t(z)
\le
\frac{D_R(z\|z_t)-D_R(z\|z_{t+1})}{\eta}
+
\frac{\eta G^2G_F^2}{2}
+
\frac{GC_\eta}{\eta}.
\]
Summing over $t$, the Bregman terms telescope.
The result follows from $D_R(z\|z_1)\le D_1$,
$D_R(z\|z_{T+1})\ge0$, and choosing a best fixed comparator.
\end{proof}

\noindent\textbf{End of Proof of Theorem~\ref{thm:main}.} 
The rest of the proof follows the same lines as the proof of \cite{ghai-lu-hazan22} using our tighter estimate obtained in Corollary~\ref{cor:approx-lemma-det}. We provide a full proof for completeness. 

\begin{proof}
Let $z_t=q(x_t)$.
At each round, define the ghost OMD step from the current
mapped OGD iterate by: 
\[
y_{t+1}
=
\argmin_{y\in\Y}
\left\{
\langle\nabla h_t(z_t),y-z_t\rangle
+
\frac1\eta D_R(y\|z_t)
\right\}.
\]
Corollary~\ref{cor:approx-lemma-det}, applied with the common
starting point $z_t=q(x_t)$, gives: 
\[
z_{t+1}=y_{t+1}+r_{t+1},
\qquad
\|r_{t+1}\|_2
\le
C_\eta:=G^5(5G_F^2+G_F^3\eta)\eta^2.
\]
Since $z_t=q(x_t)\in\Y$ for every $t$, Lemma~\ref{lem2} applies:
\begin{align*}
R_T
&\le
\frac{D_1}{\eta}
+
\frac{\eta G^2G_F^2T}{2}
+
\frac{GC_\eta T}{\eta}\\
&=
\frac{D_1}{\eta}
+
G_F^2\left(5G^6+\frac12G^2\right)\eta T
+
G^6G_F^3\eta^2T\\
&=
\frac{D_1}{\eta}
+
A\eta T
+
G^6G_F^3\eta^2T\,,
\end{align*}
where $A := G_F^2\left(5G^6+\frac12G^2\right).$ 
Substituting $\eta=\sqrt{D_1/(AT)}$ yields: 
\[
R_T
\le
2\sqrt{AD_1T}
+
\frac{G^6G_F^3D_1}{A},
\]
which proves Theorem~\ref{thm:main}.
\end{proof}

\section{Proofs for Section~\ref{sec:discussion-hessian-comp-as}}

\subsection{Proof of Proposition~\ref{prop:matrix-field-eq-Hessian}}

The proof relies on the following classical result from vector calculus. 
\begin{tcolorbox}[colframe=white!, top=2pt,left=2pt,right=2pt,bottom=2pt] 
\begin{proposition}
\label{prop:existence-grad-potential}
Let $\mathcal{U} \subseteq \mathbb{R}^d$ be convex\footnote{A simple connected domain suffices here. Convexity is enough for our purpose.} and let $F: \mathcal{U} \to \mathbb{R}^d $ be a $C^1$ vector field. Denoting $F = (F_1, \cdots, F_d)$ using its coordinate functions, if $\partial_{x_j} F_k(x) = \partial_{x_k} F_j(x)$ for all $j,k \in [d]$ and all $x \in \mathcal{U}$, then there exists a scalar function $\phi: \mathcal{U} \to \mathbb{R}$ such that $F = \nabla \phi\,.$ 
\end{proposition}
\end{tcolorbox}

\begin{proof}[Proof of Prop.~\ref{prop:matrix-field-eq-Hessian}]

Necessity is immediate: if $M=\nabla^2R$, then
\[
\partial_{z_k}M_{ij}
=
\partial_{z_kz_iz_j}^3R
=
\partial_{z_jz_iz_k}^3R
=
\partial_{z_j}M_{ik}.
\]

For each $i \in [d]$, define the vector field $M^{(i)}$ of the $i$th row of the matrix field $M$: 
\begin{equation*}
M^{(i)}(z) := (M_{i,1}(z), \cdots, M_{i,d}(z))\,.
\end{equation*}
The cross-derivatives conditions $\partial_{z_k} M_{ij}(z) = \partial_{z_j} M_{ik}(z)\,,\forall i, j, k \in [d], \forall z \in \mathcal{Y}\,,$ guarantee that $M^{(i)}$ has also matching cross-partial derivatives in the sense of Proposition~\ref{prop:existence-grad-potential}. Therefore, it follows from Proposition~\ref{prop:existence-grad-potential} that there exists a scalar function $g_i: \Y \to \mathbb{R}$ such that $\nabla g_i = M^{(i)}$ for all $i \in [d]\,.$ 

We obtain a vector field $g = (g_1, \cdots, g_d)$ whose Jacobian $J_g$ coincides with $M$, i.e. $J_g = M$ since $M^{(i)}$ are the rows of the matrix field $M$. Then, observe that $M$ is symmetric because $J_q J_q^{\top}$ is symmetric and its inverse is also symmetric. Hence, the Jacobian $J_g = M$ is also symmetric and it follows that $g$ itself has matching cross partial derivatives. Applying Proposition~\ref{prop:existence-grad-potential} again to $g$ yields the existence of a scalar function $R$ such that $g = \nabla R\,.$ We have proved that $M = J_g$ and $g = \nabla R$ which implies that $M = \nabla^2 R$. 

If $R_1$ and $R_2$ have the same Hessian on $\Y$, then
$\nabla^2(R_1-R_2)=0$. Since $\Y$ is convex, $R_1-R_2$ is affine,
which proves uniqueness up to an affine function.

Finally, $M(z)\succ0$ for every $z\in\Y$. Hence $R$ is strictly convex.
Since $\Y$ is compact and $M$ is continuous,
$\inf_{z\in\Y}\lambda_{\min}(M(z))>0$ (where $\lambda_{\min}$ stands for the smallest eigenvalue), which gives uniform strong convexity.
\end{proof}

\subsection{Detailed derivations for examples}
\label{appx:detailed-examples}

\paragraph{Example 1: Affine mixing of a separable reparameterization.}
Let
\[
    q(x)=A\,s(x)+b,
    \qquad
    s(x)=\bigl(s_1(x_1),\dots,s_d(x_d)\bigr),
\]
where \(A\in\mathbb R^{d\times d}\) is invertible, \(b\in\mathbb R^d\), and each
\(s_i:\mathbb R\to\mathbb R\) is a \(C^2\) diffeomorphism. 
Writing
\[
    D(x):=\operatorname{diag}\bigl(s_1'(x_1),\dots,s_d'(x_d)\bigr),
\]
it follows that the Jacobian of $q$ at $x$ and its outer product can be written as: 
\[
    J_q(x)=A D(x),
    \qquad
    J_q(x)J_q(x)^\top=A D(x)^2 A^\top.
\]
Therefore, for \(z=q(x)\),
\[
    M(z)
    =
    \left[J_q(q^{-1}(z))J_q(q^{-1}(z))^\top\right]^{-1}
    =
    A^{-T}D(x)^{-2}A^{-1},
    \qquad x=q^{-1}(z).
\]

Let $\xi:=A^{-1}(z-b).$ 
Since $z=A s(x)+b$, we have $\xi=s(x)$, hence \(x_i=s_i^{-1}(\xi_i)\). Define
\[
    m_i(\xi_i):=
    \frac{1}{\bigl(s_i'(s_i^{-1}(\xi_i))\bigr)^2}.
\]
Using this notation, we can rewrite $M(z)$ as follows: 
\[
    M(z)
    =
    A^{-T}\operatorname{diag}\bigl(m_1(\xi_1),\dots,m_d(\xi_d)\bigr)A^{-1},
    \qquad
    \xi=A^{-1}(z-b).
\]

This matrix field is explicitly Hessian. Indeed, define \(r_i\) by
\[
    r_i''(t)=m_i(t)
    =
    \frac{1}{\bigl(s_i'(s_i^{-1}(t))\bigr)^2},
\]
and set
\[
    R(z):=\sum_{i=1}^d r_i\bigl((A^{-1}(z-b))_i\bigr).
\]
Since \(A^{-1}(z-b)\) is affine in \(z\),
\[
    \nabla^2 R(z)
    =
    A^{-T}\operatorname{diag}\bigl(r_1''(\xi_1),\dots,r_d''(\xi_d)\bigr)A^{-1}
    =
    M(z).
\]
Thus $M$ satisfies the Hessian-compatibility condition. 
In particular, the compatibility holds although $J_q(x)=AD(x)$ is not diagonal unless $A$ is diagonal.

\paragraph{Example 2: Rank-one nonlinear mixing.}
Fix \(a\in\mathbb R^d\setminus\{0\}\) and let
\begin{equation*}
    q(x)=x+h(a^\top x)a,
\end{equation*}
where \(h:\mathbb R\to\mathbb R\) is \(C^2\). 
Writing \(r:=a^\top x\) and \(s:=\|a\|^2\), we have
\[
    J_q(x)=I+h'(r)aa^\top.
\]
This Jacobian is generally non-diagonal whenever at least two components of \(a\) are nonzero.

The matrix \(J_q(x)\) acts as multiplication by \(1+s h'(r)\) in the direction \(a\), and as the identity on \(a^\perp\). 
Assume \(1+s h'(r)\neq 0\) on the domain, so that \(q\) is locally invertible. 
Then \(J_q(x)J_q(x)^\top\) has eigenvalue \((1+s h'(r))^2\) along \(a\) and eigenvalue \(1\) on \(a^\perp\). 
Consequently, the inverse metric field has the form
\[
    M(z)=I+\beta(a^\top z)\,aa^\top,
    \qquad z=q(x),
\]
for a scalar function \(\beta\). To identify it, note that
\[
    a^\top z
    =
    a^\top q(x)
    =
    r+s h(r).
\]
Let \(g(r):=r+s h(r)\), and assume \(g\) is invertible on the relevant interval. Then
\[
    r=g^{-1}(a^\top z),
\]
and the eigenvalue of \(M(z)\) along \(a\) gives
\[
    1+s\,\beta(a^\top z)
    =
    \bigl(1+s h'(r)\bigr)^{-2}.
\]
Thus we can deduce an expression for $\beta$ as follows: 
\[
    \beta(t)
    =
    \frac{\bigl(1+s h'(g^{-1}(t))\bigr)^{-2}-1}{s}.
\]

Again \(M\) is explicitly Hessian. Let \(\psi:\mathbb R\to\mathbb R\) satisfy $\psi''(t)=\beta(t)$, 
and define: 
\[
    R(z):=\frac12\|z\|^2+\psi(a^\top z).
\]
Then, it follows from the above definition that the Hessian can be written as follows: 
\[
    \nabla^2R(z)
    =
    I+\psi''(a^\top z)aa^\top
    =
    I+\beta(a^\top z)aa^\top
    =
    M(z).
\]
Hence the Hessian-compatibility condition holds, despite the generally non-diagonal Jacobian \(J_q(x)\).

\paragraph{Verification of Assumption~\ref{as:reg-R-reparam-q-smoothness}.}
For Example~1, restrict $\X$ to a compact box
$\prod_{i=1}^d I_i$ on which each $s_i'$ is bounded away from zero and
assume that the resulting inverse metric satisfies $M(z)\succeq I$.
Then $s(\X)$ is a compact box and
$\Y=As(\X)+b$ is compact and convex. Smoothness and nonsingularity imply
uniform bounds on the first and second derivatives of $q^{-1}$, while the
explicit formula for $R$ gives bounded first and third derivatives on $\Y$.

For Example~2, write $s=\|a\|^2$ and choose
\[
\X
=
\left\{
w+\frac{r}{s}a:
w\in W,\ r\in I
\right\},
\]
where $W\subset a^\perp$ is compact and convex and $I$ is a compact interval.
If $g(r)=r+sh(r)$ satisfies $\inf_{r\in I}|g'(r)|>0$, then
\[
\Y
=
\left\{
w+\frac{g(r)}{s}a:
w\in W,\ r\in I
\right\}
\]
is also compact and convex and $q$ is a diffeomorphism from $\X$ onto $\Y$.
Assuming again $M(z)\succeq I$, the required derivative bounds follow from
compactness and nonsingularity. Finally,
\[
\nabla_wD_R(z\|w)=\nabla^2R(w)(w-z),
\]
so Assumption~\ref{as:reg-R-reparam-q-smoothness}(iii) follows from
compactness and boundedness of $\nabla^2R$.

\subsection{Proof of Theorem~\ref{thm:linear-reg-lower-bound}: Linear Regret Lower Bound for OGD}
\label{app:proof-thm-linear-reg-lower-bound}

We provide a complete proof of Theorem~\ref{thm:linear-reg-lower-bound} following the proof sketch given in Section~\ref{subsec:lin-OGD-regret-without-Hessian-comp}. 

\begin{proof}
\noindent\textbf{Step 1: A smooth square-to-square diffeomorphism.}
Let $\X=\Y=[0,1]^2$ and define
\begin{equation}
\label{eq:lb-square-map}
q(x_1,x_2)
=
\bigl(x_1,\,x_2-x_1(1-x_1)x_2(1-x_2)\bigr).
\end{equation}
Its Jacobian is
\[
J_q(x)=
\begin{pmatrix}
1&0\\
p(x)&r(x)
\end{pmatrix},
\qquad
\begin{aligned}
p(x)&:=-(1-2x_1)x_2(1-x_2),\\
r(x)&:=1-x_1(1-x_1)(1-2x_2).
\end{aligned}
\]
For every $x\in\X$, $r(x)\ge3/4$.
For each fixed $x_1$, the function $x_2\mapsto q_2(x_1,x_2)$ is
therefore strictly increasing, with $q_2(x_1,0)=0$ and $q_2(x_1,1)=1$.
Thus $q$ maps the square bijectively onto itself. The inverse function
theorem gives a smooth inverse, including smooth local extensions at
the boundary, since $r$ remains positive on a neighborhood of $\X$.
By compactness and nonsingularity, there exists $\mu>0$ such that
$J_q(x)J_q(x)^\top\succeq\mu I$ uniformly on $\X$.

Throughout this proof, $M$ denotes the inverse metric
in~\eqref{eq:matrix-field-M}. For $z=q(x)$,
\begin{equation}
\label{eq:lb-inverse-metric}
M(z)=J_q(x)^{-\top}J_q(x)^{-1}
=
\begin{pmatrix}
1+p(x)^2/r(x)^2&-p(x)/r(x)^2\\
-p(x)/r(x)^2&1/r(x)^2
\end{pmatrix}.
\end{equation}

\medskip
\noindent\textbf{Step 2: Failure of Hessian compatibility and negative circulation.}
We first compute the derivatives of the inverse metric at an interior
point, then use them to exhibit a comparator field with negative curl.
Set
$x^0=(1/2,1/2)$. Then $z^0=q(x^0)=(1/2,7/16).$
Recall from Step 1 that
\[
p(x)=-(1-2x_1)x_2(1-x_2),
\qquad
r(x)=1-x_1(1-x_1)(1-2x_2).
\]
At $x^0$, we have $p(x^0)=0$, $r(x^0)=1$, and hence
$J_q(x^0)=I$ and $M(z^0)=I$.
However, $M$ is not constant near $z^0$, so its derivatives need
not vanish.

To distinguish derivatives in the original and hidden coordinates,
write $\widehat M(x):=M(q(x))$.
By~\eqref{eq:lb-inverse-metric},
\[
\widehat M(x)
=
\begin{pmatrix}
1+p(x)^2/r(x)^2 & -p(x)/r(x)^2\\
-p(x)/r(x)^2 & 1/r(x)^2
\end{pmatrix}.
\]
Direct differentiation of $p$ and $r$ gives: 
\begin{align*}
\partial_{x_1}p(x)
&=2x_2(1-x_2),
&
\partial_{x_2}p(x)
&=-(1-2x_1)(1-2x_2),\\
\partial_{x_1}r(x)
&=-(1-2x_1)(1-2x_2),
&
\partial_{x_2}r(x)
&=2x_1(1-x_1).
\end{align*}
Therefore,
\[
\partial_{x_1}p(x^0)=\frac12,\qquad
\partial_{x_2}p(x^0)=0,\qquad
\partial_{x_1}r(x^0)=0,\qquad
\partial_{x_2}r(x^0)=\frac12.
\]

For $j\in\{1,2\}$, differentiating the entries of $\widehat M$
and then evaluating at $p(x^0)=0$ and $r(x^0)=1$ yields
\begin{align*}
\partial_{x_j}\widehat M_{11}(x^0)
&=
\left[
\frac{2p}{r^2}\,\partial_{x_j}p
-
\frac{2p^2}{r^3}\,\partial_{x_j}r
\right]_{x=x^0}
=0,\\
\partial_{x_j}\widehat M_{12}(x^0)
&=
\partial_{x_j}\widehat M_{21}(x^0)\\
&=
\left[
-\frac{1}{r^2}\,\partial_{x_j}p
+
\frac{2p}{r^3}\,\partial_{x_j}r
\right]_{x=x^0}
=-\partial_{x_j}p(x^0),\\
\partial_{x_j}\widehat M_{22}(x^0)
&=
\left[
-\frac{2}{r^3}\,\partial_{x_j}r
\right]_{x=x^0}
=-2\partial_{x_j}r(x^0).
\end{align*}

We now convert these derivatives into hidden coordinates.
Since $M(z)=\widehat M(q^{-1}(z))$, the chain rule gives
\[
\partial_{z_j}M(z^0)
=
\sum_{k=1}^2
\partial_{x_k}\widehat M(x^0)
[J_{q^{-1}}(z^0)]_{kj}.
\]
Here $J_{q^{-1}}(z^0)=J_q(x^0)^{-1}=I$, so
$\partial_{z_j}M(z^0)=\partial_{x_j}\widehat M(x^0)$.
Combining the preceding calculations, we obtain
\begin{equation}
\label{eq:lb-metric-derivatives}
\partial_{z_1}M(z^0)
=
\begin{pmatrix}
0&-1/2\\
-1/2&0
\end{pmatrix},
\qquad
\partial_{z_2}M(z^0)
=
\begin{pmatrix}
0&0\\
0&-1
\end{pmatrix}.
\end{equation}
If $M=\nabla^2R$ for a regularizer $R$, equality of mixed third
derivatives would require
$\partial_{z_2}M_{11}=\partial_{z_1}M_{12}$.
But~\eqref{eq:lb-metric-derivatives} gives
\[
\partial_{z_2}M_{11}(z^0)=0
\ne
-\frac12=\partial_{z_1}M_{12}(z^0).
\]
Thus $M$ is not Hessian-compatible.

\medskip
Next, fix the comparator $u=(0,1/2)$. 
Since the nonlinear term in $q$ vanishes when $x_1=0$,
\[
z_u:=q(u)=(0,1/2),
\qquad
z^0-z_u=(1/2,-1/16).
\]
Define
\begin{equation}
\label{eq:lb-comparator-field}
F_u(z):=M(z)(z-z_u)
=
\begin{pmatrix}
P_u(z)\\
Q_u(z)
\end{pmatrix},
\end{equation}
where $P_u$ and $Q_u$ are the first and second scalar components
of $F_u$, respectively. Explicitly,
\begin{align*}
P_u(z)
&=
M_{11}(z)z_1
+
M_{12}(z)\left(z_2-\frac12\right),\\
Q_u(z)
&=
M_{21}(z)z_1
+
M_{22}(z)\left(z_2-\frac12\right).
\end{align*}
Holding the comparator fixed and applying the product rule,
\begin{align*}
\partial_{z_1}Q_u(z)
&=
\bigl(\partial_{z_1}M_{21}(z)\bigr)z_1
+
M_{21}(z)
+
\bigl(\partial_{z_1}M_{22}(z)\bigr)
\left(z_2-\frac12\right),\\
\partial_{z_2}P_u(z)
&=
\bigl(\partial_{z_2}M_{11}(z)\bigr)z_1
+
\bigl(\partial_{z_2}M_{12}(z)\bigr)
\left(z_2-\frac12\right)
+
M_{12}(z).
\end{align*}
Using $M(z^0)=I$ and~\eqref{eq:lb-metric-derivatives}, we obtain
\begin{align*}
\partial_{z_1}Q_u(z^0)
&=
\left(-\frac12\right)\left(\frac12\right)
+0+0\left(-\frac1{16}\right)
=-\frac14,\\
\partial_{z_2}P_u(z^0)
&=
0\left(\frac12\right)
+0\left(-\frac1{16}\right)+0
=0.
\end{align*}
Consequently, the scalar curl of $F_u$ at $z^0$ is strictly negative:
\[
\partial_{z_1}Q_u(z^0)
-
\partial_{z_2}P_u(z^0)
=
-\frac14.
\]

\medskip
Since $z^0\in\operatorname{int}\Y$ and the curl is continuous,
we can choose a fixed square
\[
\mathcal Q=[a,a+L]\times[b,b+L]
\subset\operatorname{int}\Y,
\qquad z^0\in\operatorname{int}\mathcal Q,
\]
with $L>0$, and a constant $\gamma>0$ such that
\[
\partial_{z_1}Q_u(z)-\partial_{z_2}P_u(z)
\le-\gamma
\qquad\text{for every }z\in\mathcal Q.
\]
The field $F_u$ is continuously differentiable on a neighborhood
of $\mathcal Q$. Green's theorem therefore converts the circulation
around its counterclockwise-oriented boundary into the area integral
of its curl:
\begin{equation}
\label{eq:lb-negative-circulation}
\begin{aligned}
I_{\mathcal Q}
&:=
\oint_{\partial\mathcal Q}F_u(z)\cdot dz\\
&=
\oint_{\partial\mathcal Q}
\bigl(P_u(z)\,dz_1+Q_u(z)\,dz_2\bigr)\\
&=
\int_a^{a+L}\int_b^{b+L}
\bigl(\partial_{z_1}Q_u-\partial_{z_2}P_u\bigr)
(z_1,z_2)\,dz_2\,dz_1\\
&\le-\gamma L^2<0.
\end{aligned}
\end{equation}
Thus the uniformly negative curl yields a strictly negative
circulation along this fixed closed path.

Write $\kappa:=-I_{\mathcal Q}>0$ and initialize OGD at
\[
x_1=q^{-1}(a,b)\in\operatorname{int}\X,
\]
so that its initial hidden iterate is the lower-left corner
$z_1=(a,b)$ of $\mathcal Q$.
The map, square, comparator, initialization, and constants
$\gamma,\kappa$ are independent of $T$ and $\eta$.

\medskip
\noindent\textbf{Step 3: Exact steering around a closed grid.}
Fix $0<\eta\le L$ where we recall here that $L$ is the width of the square $\mathcal{Q}$ and define
\[
n:=\left\lceil\frac L\eta\right\rceil,\qquad
\Delta:=\frac Ln,\qquad N:=4n.
\]
Then $\eta/2\le\Delta\le\eta$ and $N\le8L/\eta$.
Starting from the point $(a,b)$, prescribe $n$ steps in direction $e_1$, then
$n$ in direction $e_2$, then $n$ in direction $-e_1$, and finally
$n$ in direction $-e_2$, each of length $\Delta$.
This grid traverses $\partial\mathcal Q$ and returns exactly to $(a,b)$.
For each such round, let $v_t \in \{\pm e_1, \pm e_2 \}$ denote the prescribed direction and set: 
\begin{equation*}
    z_t^+:=z_t+\Delta v_t\,.
\end{equation*}
Then, choose:
\begin{equation}
\label{eq:lb-exact-steering}
s_t:=\frac1\eta J_q(x_t)^{-\top}
\bigl[x_t-q^{-1}(z_t^+)\bigr],
\qquad
h_t(z):=\langle s_t,z\rangle,
\qquad \ell_t:=h_t\circ q.
\end{equation}
Once selected, $s_t$ is held fixed in the definition of $h_t$.
Thus $\nabla\ell_t(x)=J_q(x)^\top s_t$, and
\[
x_t-\eta\nabla\ell_t(x_t)
=x_t-\eta J_q(x_t)^\top s_t
=q^{-1}(z_t^+)\in\operatorname{int}\X.
\]
We briefly justify why $q^{-1}(z_t^+)\in\operatorname{int}\X$. 
Since the nonlinear term in $q$ vanishes on the boundary of the square,
$q(x)=x$ for every $x\in\partial\X=\partial\Y$.
Thus the preimage of $z_t^+\in\operatorname{int}\Y$ cannot lie on
$\partial\X$. 
The unprojected update is therefore already feasible.
Since Euclidean projection leaves every feasible point unchanged,
\[
x_{t+1}
=
\Pi_\X\bigl(q^{-1}(z_t^+)\bigr)
=
q^{-1}(z_t^+),
\qquad
z_{t+1}=q(x_{t+1})=z_t^+.
\]
Hence projection is inactive and the prescribed hidden-space update
is followed exactly. 
Induction proves that the prescribed grid is followed exactly and
can be repeated without any accumulated discretization drift.
Only the adversary uses $q$; the learner still performs ordinary
projected OGD.

The loss vectors are bounded uniformly in $\eta$ and $T$.
Indeed, let
$K:=\sup_{z\in\Y}\|J_{q^{-1}}(z)\|_{\rm op}<\infty$.
Since the segment from $z_t$ to $z_t^+$ lies in $\Y$,
\[
\|s_t\|_2
\le\frac K\eta\|q^{-1}(z_t)-q^{-1}(z_t^+)\|_2
\le K^2\frac\Delta\eta\le K^2.
\]
Consequently, both the losses and their original-space gradients are
uniformly bounded on $\X$, independently of $\eta$ and $T$.

\medskip
\noindent\textbf{Step 4: Regret as a discrete line integral.}
We now express the instantaneous regret against the fixed comparator
$u$ as a discrete line integral of the field
$F_u(z)=M(z)(z-z_u)$, where $z_u=q(u)$.

Recall that the exact-steering construction gives
\[
z_{t+1}=z_t^+=z_t+\Delta v_t,
\qquad
s_t
=
\frac1\eta J_q(x_t)^{-\top}
\bigl[x_t-q^{-1}(z_t+\Delta v_t)\bigr].
\]
To relate this exact choice of $s_t$ to the inverse metric $M(z_t)$,
we expand $q^{-1}$ around $z_t$.
Let $B<\infty$ be a uniform bound on the second derivative of
$q^{-1}$ on $\Y$, in operator norm.
Since $q^{-1}(z_t)=x_t$,
$J_{q^{-1}}(z_t)=J_q(x_t)^{-1}$, and $\|v_t\|_2=1$,
Taylor's theorem along the grid segment gives
\[
q^{-1}(z_t+\Delta v_t)
=
x_t+\Delta J_q(x_t)^{-1}v_t+\rho_t,
\qquad
\|\rho_t\|_2\le\frac B2\Delta^2.
\]
Substituting this expansion into the definition of $s_t$, we obtain
\begin{equation*}
s_t
=
-\frac{\Delta}{\eta}
J_q(x_t)^{-\top}J_q(x_t)^{-1}v_t
-\frac1\eta J_q(x_t)^{-\top}\rho_t
=
-\frac{\Delta}{\eta}M(z_t)v_t+e_t,
\end{equation*}
where 
$
e_t:=-\frac1\eta J_q(x_t)^{-\top}\rho_t.$ 

Recall that
$K:=\sup_{z\in\Y}\|J_{q^{-1}}(z)\|_{\rm op}<\infty$.
Since
$\|J_q(x_t)^{-\top}\|_{\rm op}
=\|J_{q^{-1}}(z_t)\|_{\rm op}\le K$,
the vector error satisfies
\[
\|e_t\|_2
\le
\frac1\eta\|J_q(x_t)^{-\top}\|_{\rm op} \cdot \|\rho_t\|_2
\le
\frac{KB}{2}\frac{\Delta^2}{\eta}.
\]

We next use this decomposition of $s_t$ to compute the loss
difference against $u$.
By the definitions $\ell_t=h_t\circ q$ and
$h_t(z)=\langle s_t,z\rangle$,
\begin{align*}
\ell_t(x_t)-\ell_t(u)
&=
h_t(q(x_t))-h_t(q(u))\\
&=
h_t(z_t)-h_t(z_u)\\
&=
\langle s_t,z_t-z_u\rangle\\
&=
-\frac{\Delta}{\eta}
\langle M(z_t)v_t,z_t-z_u\rangle
+
\langle e_t,z_t-z_u\rangle\\
&=
-\frac{\Delta}{\eta}
\langle M(z_t)(z_t-z_u),v_t\rangle
+
\langle e_t,z_t-z_u\rangle\\
&=
-\frac{\Delta}{\eta}\langle F_u(z_t),v_t\rangle
+
\langle e_t,z_t-z_u\rangle,
\end{align*}
where the penultimate equality uses symmetry of $M(z_t)$.

Define the scalar error
\[
\epsilon_t:=\langle e_t,z_t-z_u\rangle.
\]
Both $z_t$ and $z_u$ belong to $\Y=[0,1]^2$, whose Euclidean
diameter is $\sqrt2$. Hence, by Cauchy--Schwarz,
\[
|\epsilon_t|
\le
\|e_t\|_2 \cdot \|z_t-z_u\|_2
\le
\frac{KB\sqrt2}{2}\frac{\Delta^2}{\eta}.
\]
Finally, using the exact displacement
$z_{t+1}-z_t=\Delta v_t$, we conclude that
\begin{equation}
\label{eq:lb-regret-line-sum}
\ell_t(x_t)-\ell_t(u)
=
-\frac1\eta
\langle F_u(z_t),z_{t+1}-z_t\rangle+\epsilon_t,
\qquad
|\epsilon_t|\le C_0\frac{\Delta^2}{\eta},
\end{equation}
where $C_0:=KB\sqrt2/2$ is independent of $\eta$ and $T$.

Thus, over any prefix of $\tau$ steering rounds, the cumulative
regret against $u$ satisfies
\[
R_\tau(u)
:=
\sum_{t=1}^{\tau}\bigl(\ell_t(x_t)-\ell_t(u)\bigr)
=
-\frac1\eta
\sum_{t=1}^{\tau}
\langle F_u(z_t),z_{t+1}-z_t\rangle
+
E_\tau,
\]
where
\[
E_\tau:=\sum_{t=1}^{\tau}\epsilon_t,
\qquad
|E_\tau|
\le
C_0\tau\frac{\Delta^2}{\eta}
\le
C_0\eta\tau,
\]
using $\Delta\le\eta$.
In particular, one complete cycle has $N=4n=4L/\Delta$ rounds,
so its cumulative Taylor error is bounded uniformly:
\[
|E_N|
\le
C_0N\frac{\Delta^2}{\eta}
=
4LC_0\frac{\Delta}{\eta}
\le
4LC_0.
\]
The remaining task is to compare the discrete line integral over
one cycle with the negative circulation~$I_{\mathcal Q}$.

\medskip
\noindent\textbf{Step 5: A uniform positive regret per cycle.}
We first compare the discrete line integral over one cycle with the
continuous circulation $I_{\mathcal Q}$.
Since $F_u$ is continuously differentiable on a neighborhood of
the fixed compact square $\mathcal Q$, it has a Lipschitz constant
$L_F<\infty$ on $\mathcal Q$, independent of $\eta$ and $T$.

Consider one grid segment from $z_t$ to $z_{t+1}$, and write
\[
d_t:=z_{t+1}-z_t=\Delta v_t,
\qquad
\|d_t\|_2=\Delta.
\]
Parameterizing this oriented segment by $z_t+s d_t$, with
$s\in[0,1]$, its exact line integral is
\[
I_t
:=
\int_{[z_t,z_{t+1}]}F_u(z)\cdot dz
=
\int_0^1
\langle F_u(z_t+s d_t),d_t\rangle\,ds.
\]
The corresponding term in our discrete line integral is
$\langle F_u(z_t),d_t\rangle$: it replaces the vector field along
the segment by its value at the starting point.
Since this term is constant with respect to $s$,
\[
I_t-\langle F_u(z_t),d_t\rangle
=
\int_0^1
\langle F_u(z_t+s d_t)-F_u(z_t),d_t\rangle\,ds.
\]
Using Cauchy--Schwarz's inequality and Lipschitzness of $F_u$, we obtain
\begin{align*}
\left|I_t-\langle F_u(z_t),d_t\rangle\right| 
\le
\int_0^1
\|F_u(z_t+s d_t)-F_u(z_t)\|_2\cdot\|d_t\|_2\,ds
\le
\int_0^1
L_Fs\|d_t\|_2^2\,ds
=
\frac{L_F}{2}\Delta^2.
\end{align*}
Thus the approximation error on each segment is of order
$\Delta^2$.

The $N=4n$ grid segments traverse the boundary of $\mathcal Q$
exactly once in the counterclockwise direction, with
$z_{N+1}=z_1$. By additivity of line integrals,
\[
\sum_{t=1}^{N}I_t
=
\oint_{\partial\mathcal Q}F_u(z)\cdot dz
=
I_{\mathcal Q}.
\]
Consequently, the triangle inequality and the preceding
segment-wise bound give
\begin{align*}
\left|
\sum_{t=1}^{N}
\langle F_u(z_t),z_{t+1}-z_t\rangle
-
I_{\mathcal Q}
\right|
&=
\left|
\sum_{t=1}^{N}
\bigl(\langle F_u(z_t),d_t\rangle-I_t\bigr)
\right|\\
&\le
\sum_{t=1}^{N}
\left|\langle F_u(z_t),d_t\rangle-I_t\right|\\
&\le
\frac{NL_F}{2}\Delta^2\\
&=
2LL_F\Delta,
\end{align*}
where the last equality uses $N=4n$ and $n\Delta=L$.
Hence the discrete line integral approximates the circulation
with an $O(\Delta)$ error, with a constant independent of
$\eta$ and $T$.
Combining this with~\eqref{eq:lb-regret-line-sum} and $n\Delta=L$,
\begin{align*}
\sum_{t=1}^{N}\bigl(\ell_t(x_t)-\ell_t(u)\bigr)
&\ge\frac\kappa\eta
-\frac{2LL_F\Delta}{\eta}
-\frac{4nC_0\Delta^2}{\eta}\\
&=\frac\kappa\eta
-(2LL_F+4LC_0)\frac\Delta\eta\\
&\ge\frac\kappa\eta-C_1,
\end{align*}
where $C_1:=2LL_F+4LC_0$ is independent of $\eta$ and $T$.
Choose
\[
\eta_0:=\min\left\{L,\frac{\kappa}{2(C_1+1)}\right\}>0.
\]
For every $0<\eta\le\eta_0$, each cycle therefore contributes at least
$\kappa/(2\eta)$ regret against $u$.

\medskip
\noindent\textbf{Step 6: Repetition over the horizon.}
For a given horizon $T$, perform
$m:=\lfloor T/N\rfloor$ complete cycles and set $s_t=0$ on all
remaining rounds. Each complete cycle starts at the same point and
obeys the same lower bound. Since external regret is at least regret
against the fixed comparator $u$,
\begin{equation*}
R_T
\ge\sum_{t=1}^{T}\bigl(\ell_t(x_t)-\ell_t(u)\bigr) 
\ge m\frac{\kappa}{2\eta} 
\ge\left(\frac TN-1\right)\frac{\kappa}{2\eta} 
\ge\frac{\kappa}{16L}T-\frac{\kappa}{2\eta},
\end{equation*}
where the last inequality uses $N\le8L/\eta$.
Thus the theorem holds with $c=\kappa/(16L)$, $C=\kappa/2$, and
$S=K^2$. In particular, a horizon-dependent constant stepsize
$\eta=\Theta(T^{-1/2})$ gives $R_T\ge cT-O(\sqrt T)=\Omega(T)$
for sufficiently large $T$.
\end{proof}

\section{Proof of Theorem~\ref{thm:bandit}: Expected Regret Bound in the Bandit Setting}
\label{appx:proofs-bandit}

Recall first that the regret of bandit OGD (Algorithm~\ref{alg4}) is defined for any time horizon $T \geq 1$ by: 
\begin{equation}
\label{eq:regret-bogd}
R_T := \sum_{t=1}^T \ell_t(\hat{x}_t) - \min_{x \in \X} \sum_{t=1}^T \ell_t(x)\, 
= \sum_{t=1}^T h_t(\hat{z}_t) - \min_{z \in \Y} \sum_{t=1}^T h_t(z)\,,
\end{equation}
where the last equality follows from using the notation $\hat{z}_t := q(\hat{x}_t)$ for any $t \geq 1$ and recalling that $\ell_t = h_t \circ q$ with $q$ being a diffeomorphism. 

To prove Theorem~\ref{thm:bandit}, we first decompose the expected regret into four different error terms.

\begin{tcolorbox}[colframe=white!, top=2pt,left=2pt,right=2pt,bottom=2pt] 
\begin{lemma}[Expected regret bound decomposition]
\label{lem:expected-reg-decomp}
Fix a time horizon $T \geq 1$.  Let $z^{\star} \in \argmin_{z \in \Y} \sum_{t=1}^T h_t(z), z^{\star}_{\delta} \in \argmin_{z \in \Y_{\delta}} \sum_{t=1}^T h_t(z)\,$ and let $\mathcal{F}_{t} := \sigma(\zeta_1, \dots, \zeta_t)$ be the filtration of $\sigma$-algebras induced by the random variables $(\zeta_t)$ used in bandit OGD. Then the expected regret $\mathbb{E}[R_T]$ of bandit OGD (Algorithm~\ref{alg4}) can be upper-bounded as follows: 
\begin{equation*}
\mathbb{E}[R_T] \leq R_T^{(1)} +  R_T^{(2)} + R_T^{(3)} + R_T^{(4)}\,,
\end{equation*}
where each error term in the right-hand side above is defined as follows: 
\begin{align}
R_T^{(1)} &:= \mathbb{E}\left[\sum_{t=1}^T \langle \tilde{g}_t, z_t - z^{\star}_{\delta} \rangle \right]\,,  &&\hfill \text{(perturbed OMD regret for stochastic linearized losses)} \label{eq:RT1}\\
R_T^{(2)} &:= \sum_{t=1}^T h_t(z^{\star}_{\delta}) - h_t(z^{\star})\,,  &&\hfill \text{(domain shrinkage error)}\label{eq:RT2}\\
R_T^{(3)} &:= \mathbb{E}\left[\sum_{t=1}^T \langle b_t, z^{\star}_{\delta} - z_t \rangle \right]\,,  &&\hfill \text{(cumulative gradient estimation bias)}\label{eq:RT3}\\
b_t &:= \mathbb{E}[\tilde{g}_t |\mathcal{F}_{t-1}] - \nabla h_t(z_t)\,,  &&\hfill \text{(gradient estimation bias)}\\
R_T^{(4)} &:= \mathbb{E}\left[ \sum_{t=1}^T h_t(\hat{z}_t) - h_t(z_t)\right]  &&\hfill \text{(smoothing error)} \label{eq:RT4}\,.
\end{align}
\end{lemma}
\end{tcolorbox}

\begin{proof}
The loss sequence is fixed in advance, and the initialization is
deterministic. We choose fixed minimizers $z^\star$ and
$z_\delta^\star$, independently of the random directions.
In particular, $z_t-z_\delta^\star$ is
$\mathcal F_{t-1}$-measurable.

Using the regret definition~\eqref{eq:regret-bogd}, we have the following decomposition: 
\begin{equation}
\label{eq:regret-decomp1}
\mathbb{E}[R_T] = \mathbb{E}\left[\sum_{t=1}^T h_t(\hat{z}_t) - h_t(z^{\star}_{\delta})\right] 
+ \sum_{t=1}^T h_t(z^{\star}_{\delta}) - h_t(z^{\star})
= \mathbb{E}\left[\sum_{t=1}^T h_t(\hat{z}_t) - h_t(z^{\star}_{\delta})\right] + R_T^{(2)}\,.
\end{equation}
To control regret against the shrunk comparator $z_{\delta}^{\star}$, we start by decomposing the instantaneous regret. Since $z_t$ is $\mathcal{F}_{t-1}$-measurable, we have: 
\begin{equation}
\label{eq:1-reg-decomp}
\mathbb{E}[ h_t(\hat{z}_t) - h_t(z^{\star}_{\delta}) |\mathcal{F}_{t-1}] 
= \mathbb{E}[ h_t(\hat{z}_t) - h_t(z_t) |\mathcal{F}_{t-1}] 
+ h_t(z_t) -  h_t(z_{\delta}^{\star})\,.
\end{equation}
Then by convexity of $h_t$, we have: 
\begin{equation}
\label{eq:2-reg-decomp}
h_t(z_t) -  h_t(z_{\delta}^{\star}) \leq \langle \nabla h_t(z_t), z_t- z_{\delta}^{\star} \rangle\,.
\end{equation}
Recall now that we denote by $b_t$ the bias due to one-point smoothing defined by: 
\begin{equation*}
b_t = \mathbb{E}[\tilde{g}_t|\mathcal{F}_{t-1}] - \nabla h_t(z_t)\,.
\end{equation*}
Note here that $\tilde{g}_t := J_q(x_t)^{-\top} g_t$ is not computed and just used here in the analysis whereas $g_t$ is the one-point estimator actually computed in BOGD. 
It follows then that: 
\begin{equation}
\label{eq:3-reg-decomp}
\langle \nabla h_t(z_t), z_t - z_{\delta}^{\star} \rangle 
= \langle \mathbb{E}[\tilde{g}_t|\mathcal{F}_{t-1}], z_t - z_{\delta}^{\star}\rangle - \langle b_t, z_t - z_{\delta}^{\star}\rangle\,.
\end{equation}
Combining \eqref{eq:1-reg-decomp}, \eqref{eq:2-reg-decomp} and \eqref{eq:3-reg-decomp}, taking expectations and noting that $z_t - z_{\delta}^{\star}$ is $\mathcal{F}_{t-1}$-measurable, we obtain:
\begin{equation*}
\mathbb E[h_t(\hat z_t)-h_t(z_\delta^\star)]
\le
\mathbb E[\langle\tilde g_t,z_t-z_\delta^\star\rangle]
+\mathbb E[\langle b_t,z_\delta^\star-z_t\rangle]
+ \mathbb E[h_t(\hat z_t)-h_t(z_t)].
\end{equation*} 

Summing the above inequality over $t$ yields: 
\begin{align}
\label{eq:4-reg-decomp}
\mathbb E\left[
\sum_{t=1}^T
\bigl(h_t(\hat z_t)-h_t(z_\delta^\star)\bigr)
\right]
&\le
\mathbb E\left[
\sum_{t=1}^T
\langle\tilde g_t,z_t-z_\delta^\star\rangle
\right]
+\mathbb E\left[
\sum_{t=1}^T
\langle b_t,z_\delta^\star-z_t\rangle
\right]
+
\mathbb E\left[
\sum_{t=1}^T
\bigl(h_t(\hat z_t)-h_t(z_t)\bigr)
\right]\nonumber\\
&=
R_T^{(1)}+R_T^{(3)}+R_T^{(4)}.
\end{align}
Finally, combining \eqref{eq:4-reg-decomp} with \eqref{eq:regret-decomp1} gives the desired bound: 
\begin{equation*}
\mathbb{E}[R_T] \leq R_T^{(1)} + R_T^{(2)} + R_T^{(3)} + R_T^{(4)}\,.
\end{equation*}
\end{proof}

Given the regret decomposition of Lemma~\ref{lem:expected-reg-decomp}, we control each of the error terms $R_T^{(i)}, i \in \{ 1,2,3,4\}$ defined in \eqref{eq:RT1} to \eqref{eq:RT4} respectively, in a dedicated lemma.   

\begin{tcolorbox}[colframe=white!, top=2pt,left=2pt,right=2pt,bottom=2pt] 
\begin{lemma}[Domain shrinkage error]
\label{lem:domain-shrinkage-error}
Let $T \geq 1$, let $z^{\star} \in \argmin_{z \in \Y} \sum_{t=1}^T h_t(z)$, and $z^{\star}_{\delta} \in \argmin_{z \in \Y_{\delta}} \sum_{t=1}^T h_t(z)\,.$ Then we have: 
\begin{equation*}
R_T^{(2)} = \sum_{t=1}^T h_t(z_{\delta}^{\star}) - h_t(z^{\star}) \leq \delta G_FDT\,,
\end{equation*}
where we recall that $\delta$ is the smoothing parameter in the one-point zeroth-order gradient estimator, $G_F$ is the uniform Lipschitz constant of $\ell_t$ and $D$ is the finite diameter of the compact set $\X$. 
\end{lemma}
\end{tcolorbox}

\begin{proof}
Since $\Y_{\delta} = q(\X_{\delta})$, $q$ is a diffeomorphism,  $\X_{\delta} = (1-\delta) \X$ and $\ell_t = h_t \circ q$, the comparator term over the shrunk set $\X_{\delta}$ can be rewritten as follows: 
\begin{equation}
\label{eq:min-eq}
\min_{z \in \Y_{\delta}} \sum_{t=1}^T h_t(z) = \min_{x \in \X_{\delta}} \sum_{t=1}^T \ell_t(x) = \min_{x \in \X} \sum_{t=1}^T \ell_t((1-\delta) x)\,.
\end{equation}
Then for any $x \in \X,$ we have: 
\begin{equation*}
\sum_{t=1}^T \ell_t((1-\delta)x)
\leq \sum_{t=1}^T |\ell_t((1-\delta)x) - \ell_t(x)| + \sum_{t=1}^T \ell_t(x) \leq \delta G_F DT + \sum_{t=1}^T h_t(q(x))\,,
\end{equation*}
where the last inequality uses uniform $G_F$-Lipschitzness of the loss functions $\ell_t$ and boundedness of the set $\X$ which has finite diameter $D$. 

Taking the minimum over $x \in \X$, using \eqref{eq:min-eq} and recalling that $z^{\star} \in \argmin_{z \in \Y} \sum_{t=1}^T h_t(z)$ and $z^{\star}_{\delta} \in \argmin_{z \in \Y_{\delta}} \sum_{t=1}^T h_t(z)$ yields the desired inequality: 
\begin{equation*}
R_T^{(2)} = \sum_{t=1}^T h_t(z_{\delta}^{\star}) - h_t(z^{\star}) \leq \delta G_FDT\,.
\end{equation*}
\end{proof}

\begin{tcolorbox}[colframe=white!,
top=2pt,left=2pt,right=2pt,bottom=2pt]
\begin{lemma}[Gradient estimation bias]
\label{lem:grad-estimation-bias}
For any $T \geq 1$, we have: 
\begin{equation*} 
|R_T^{(3)}| = \left| \mathbb{E}\left[ \sum_{t=1}^T \langle b_t, z_{\delta}^{\star} - z_t  \rangle \right]\right|   \leq \delta G^2 H DT\,,
\end{equation*}
where we recall that $H$ is the uniform smoothness constant of the loss $\ell_t.$
\end{lemma}
\end{tcolorbox}

\begin{proof}
First, using the Cauchy--Schwarz inequality, it follows that: 
\begin{equation}
\label{eq:ics-bias}
|R_T^{(3)}| = \left| \mathbb{E}\left[ \sum_{t=1}^T \langle b_t, z_{\delta}^{\star} - z_t  \rangle \right]\right|
\leq \mathbb{E}\left[ \sum_{t=1}^T \|b_t\| \cdot \|z_t - z_{\delta}^{\star}\| \right]\,. 
\end{equation}

We control each one of the terms in the right-hand side separately.

Since $z_t = q(x_t)$ and $z_{\delta}^{\star} = q(x_{\delta}^{\star})$ for some  $x_t \in \X$ and $x_{\delta}^{\star} \in \X_{\delta} \subset \X$, we have: 
\begin{equation}
\label{eq:bias-proof2}
\|z_t - z_{\delta}^{\star}\| = \|q(x_t) - q(x_{\delta}^{\star})\| \leq G \|x_t - x_{\delta}^{\star}\| \leq GD\,,
\end{equation}
where the first inequality follows from using $G$-Lipschitzness of $q$ (Assumption~\ref{as:reg-R-reparam-q-smoothness}) and the second one uses the definition of the diameter of $\X$ supposed to be finite.  

We now control the norm of the bias $\|b_t\|$ for any $t$. 
For $x\in\X_\delta$, define the ball-smoothed loss: 
\begin{equation*}
\bar{\ell}_t(x) := \mathbb{E}_{v \sim \mathcal{U}(\mathbb{B})}[\ell_t(x + \delta v)]\,
\end{equation*}
where $v$ is a random variable with uniform distribution $\mathcal{U}(\mathbb{B})$ on the unit ball $\mathbb{B}$. This loss is well defined: if $x=(1-\delta)x'$ with $x'\in\X$
and $v\in\mathbb B=B_2\subseteq\X$, then
$x+\delta v=(1-\delta)x'+\delta v\in\X$ by convexity. 
It is known from \cite[Lemma~1]{flaxman-kalai-mcmahan05} that the conditional expectation of the spherical estimator $g_t$ is equal to the gradient of the ball-smoothed loss: 
\begin{equation*}
\mathbb{E}[g_t |\mathcal{F}_{t-1}] = \nabla \bar{\ell}_t(x_t)\,.
\end{equation*}
The (hidden-space ghost) gradient estimator of $\nabla h_t(z_t)$ is 
\begin{equation*}
\tilde{g}_t = J_q(x_t)^{-\top} g_t\,.
\end{equation*}
(Recall that $\nabla \ell_t(x_t) = J_q(x_t)^{\top} \nabla h_t(z_t)$ by the chain rule.)
Since $x_t$ is $\mathcal{F}_{t-1}$-measurable, it follows that: 
\begin{equation}
\label{eq:expectation-tilde-g}
\mathbb{E}[\tilde{g}_t|\mathcal{F}_{t-1}] = J_q(x_t)^{-\top} \mathbb{E}[g_t|\mathcal{F}_{t-1}] = J_q(x_t)^{-\top} \nabla \bar{\ell}_t(x_t)\,.
\end{equation}
Moreover, we have by the chain rule $\nabla \ell_t(x_t) = J_q(x_t)^{\top} \nabla h_t(z_t)$ which implies that $\nabla h_t(z_t)= J_q(x_t)^{-\top} \nabla \ell_t(x_t).$ Therefore, using this identity together with \eqref{eq:expectation-tilde-g}, we can express the bias $b_t$ as follows: 
\begin{align*}
b_t &= \mathbb{E}[\tilde{g}_t|\mathcal{F}_{t-1}] - \nabla h_t(z_t)\nonumber\\  
&= J_q(x_t)^{-\top} (\nabla \bar{\ell}_t(x_t) - \nabla \ell_t(x_t)) \nonumber\\ 
&= J_q(x_t)^{-\top} \mathbb{E}_{v \sim \mathcal{U}(\mathbb{B})}[\nabla \ell_t(x_t + \delta v) - \nabla \ell_t(x_t) ]\,.
\end{align*}
Taking the norm, using boundedness of the inverse Jacobian of $q$ (Assumption~\ref{as:reg-R-reparam-q-smoothness}) and smoothness of the loss function $\ell_t$ (Assumption~\ref{as:smoothness-bandit}), we have: 
\begin{equation}
\label{eq:bias-proof3}
\begin{aligned}
\|b_t\|_2
&\le
\|J_q(x_t)^{-\top}\|_{\rm op}\,
\mathbb E_{v\sim\mathcal U(\mathbb B)}
\left[
\|\nabla\ell_t(x_t+\delta v)-\nabla\ell_t(x_t)\|_2
\right]\\
&\le
\delta GH\,
\mathbb E_{v\sim\mathcal U(\mathbb B)}[\|v\|_2]\\
&\le
\delta GH.
\end{aligned}
\end{equation} 
Combining \eqref{eq:bias-proof3} and \eqref{eq:bias-proof2} in \eqref{eq:ics-bias}, we obtain the desired inequality: 
\begin{equation*}
|R_T^{(3)}| \leq \delta G^2 H D T\,.
\end{equation*}
\end{proof}

\begin{tcolorbox}[colframe=white!,
top=2pt,left=2pt,right=2pt,bottom=2pt]
\begin{lemma}[Smoothing error]
\label{lem:smoothing-error}
For any $T \geq 1$, we have: 
\begin{equation*}
|R_T^{(4)}| \leq \left| \mathbb{E}\left[ \sum_{t=1}^T h_t(\hat{z}_t) - h_t(z_t) \right] \right| \leq \delta G_F T\,.
\end{equation*}
\end{lemma}
\end{tcolorbox}

\begin{proof}
The desired bound follows from the following: 
\begin{align}
|R_T^{(4)}| &= \left| \mathbb{E}\left[ \sum_{t=1}^T h_t(\hat{z}_t) - h_t(z_t) \right] \right|\nonumber\\
&= \left| \mathbb{E}\left[ \sum_{t=1}^T \ell_t(\hat{x}_t) - \ell_t(x_t) \right] \right|&&\hfill \text{($\ell_t = h_t \circ q$, $\hat{z}_t = q(\hat{x}_t), z_t = q(x_t)$)}\nonumber\\
&\leq G_F \mathbb{E}\left[ \sum_{t=1}^T \|\hat{x}_t - x_t\| \right]
&&\hfill \text{(by $G_F$-Lipschitzness of $\ell_t$ (Assumption~\ref{as:boundedness-grad-bregman}))}\nonumber\\
&= G_F \mathbb{E}\left[ \sum_{t=1}^T \delta \|\zeta_t\| \right] &&\hfill \text{($\hat{x}_t = x_t + \delta \zeta_t$ in BOGD (Algorithm~\ref{alg4}))}\nonumber\\
&= \delta G_F T&&\hfill \text{($\|\zeta_t\| = 1$)}\,.
\end{align}
\end{proof}

\begin{tcolorbox}[colframe=white!,
top=2pt,left=2pt,right=2pt,bottom=2pt]
\begin{lemma}[Perturbed ghost OMD regret for random linear losses]
\label{lem:perturbed-omd-reg-random-lin-loss}
Suppose $(z_t)_{t=1}^{T+1}\subseteq\Y_\delta$ satisfies: 
\begin{equation*}
z_{t+1} = r_{t+1} + y_{t+1}\,, \quad y_{t+1} = \argmin_{y \in \Y_{\delta}} \{ \tilde{g}_t^{\top} (y - z_t) + \frac{1}{\eta} D_R(y||z_t)\}\,.
\end{equation*}
Then for every $z_{\delta}^{\star} \in \Y_{\delta},$
\begin{equation*}
\sum_{t=1}^T \langle \tilde{g}_t, z_t - z_{\delta}^{\star} \rangle 
\leq \frac{D_R(z_{\delta}^{\star}||z_1)}{\eta} + \frac{\eta}{2} \sum_{t=1}^T \|\tilde{g}_t\|_2^2 + \frac{G}{\eta} \sum_{t=1}^T \|r_{t+1}\|_2\,.
\end{equation*}
If in addition $\|\tilde{g}_t\| \leq \tilde{G}_F$ and $\|r_{t+1}\| \leq \hat{C}_{\eta}$, then: 
\begin{equation}
\label{eq:R1T-bandit}
R_T^{(1)} = \mathbb{E}\left[ \sum_{t=1}^T \langle \tilde{g}_t, z_t - z_{\delta}^{\star} \rangle\right] \leq \frac{D_1}{\eta} + \frac{\eta \tilde{G}_F^2T}{2} + \frac{\hat{C}_{\eta} G T}{\eta}\,.
\end{equation}
\end{lemma}
\end{tcolorbox}

\begin{proof}
First, we prove the following inequality which follows from using the first-order optimality condition (from the ghost OMD update) together with Bregman three-point identity. For any comparator $z_{\delta}^{\star} \in \Y_{\delta}$, we have: 
\begin{equation}
\label{eq:omd-ineq}
\langle \tilde{g}_t, z_t - z_{\delta}^{\star}\rangle  
\leq \frac{D_R(z_{\delta}^{\star}||z_t) - D_R(z_{\delta}^{\star}||y_{t+1})}{\eta} + \frac{\eta}{2}\|\tilde{g}_t\|_2^2\,.
\end{equation}

We now provide the proof of \eqref{eq:omd-ineq}. By definition of $y_{t+1}$, using the first-order optimality condition gives for every $z_{\delta}^{\star} \in \Y_{\delta}$, 
\begin{equation*}
\left\langle \tilde{g}_t + \frac{1}{\eta} (\nabla R(y_{t+1}) - \nabla R(z_t)), z_{\delta}^{\star} - y_{t+1} \right\rangle \geq 0\,.
\end{equation*}
Rearranging this inequality, we obtain: 
\begin{equation}
\label{eq:proof-omd1}
\langle \tilde{g}_t, y_{t+1} - z_{\delta}^{\star} \rangle 
\leq \frac{1}{\eta} \langle \nabla R(y_{t+1}) - \nabla R(z_t), z_{\delta}^{\star} - y_{t+1} \rangle\,.
\end{equation}
Using the three-point identity, we have: 
\begin{equation}
\label{eq:proof-omd2}
\langle \nabla R(y_{t+1}) - \nabla R(z_t), z_{\delta}^{\star} - y_{t+1} \rangle = D_R(z_{\delta}^{\star}||z_t) - D_R(z_{\delta}^{\star}||y_{t+1}) - D_R(y_{t+1}||z_t)\,.
\end{equation}
Combining \eqref{eq:proof-omd1} and \eqref{eq:proof-omd2} yields: 
\begin{equation*}
\langle \tilde{g}_t, y_{t+1} - z_{\delta}^{\star} \rangle 
\leq \frac{1}{\eta}\left[ D_R(z_{\delta}^{\star}||z_t) - D_R(z_{\delta}^{\star}||y_{t+1}) - D_R(y_{t+1}||z_t) \right]\,.
\end{equation*}
Using the above inequality, we obtain: 
\begin{align*}
\langle \tilde{g}_t, z_t - z_{\delta}^{\star} \rangle 
&= \langle \tilde{g}_t, z_t - y_{t+1} \rangle 
+ \langle \tilde{g}_t, y_{t+1} - z_{\delta}^{\star} \rangle \nonumber\\ 
&\leq \langle \tilde{g}_t, z_t - y_{t+1} \rangle 
+ \frac{1}{\eta}\left[ D_R(z_{\delta}^{\star}||z_t) - D_R(z_{\delta}^{\star}||y_{t+1}) - D_R(y_{t+1}||z_t) \right] \nonumber\\
&= \langle \tilde{g}_t, z_t - y_{t+1} \rangle - \frac{1}{\eta}D_R(y_{t+1}||z_t) 
+ \frac{1}{\eta}\left[ D_R(z_{\delta}^{\star}||z_t) - D_R(z_{\delta}^{\star}||y_{t+1}) \right]\,.
\end{align*}
Using $1$-strong convexity of $R$ and Cauchy-Schwarz inequality yields: 
\begin{equation*}
\langle \tilde{g}_t, z_t - y_{t+1} \rangle - \frac{1}{\eta}D_R(y_{t+1}||z_t)  \leq \|\tilde{g}_t\|_2 \cdot \|z_t - y_{t+1}\|_2 - \frac{1}{2\eta} \|z_t - y_{t+1}\|_2^2\,. 
\end{equation*}
Note now that for any $a \geq 0, r \geq 0, ar - \frac{1}{2\eta} r^2 \leq \frac{\eta}{2}a^2\,.$ Hence with $a = \|\tilde{g}_t\|_2$ and $r = \|z_t - y_{t+1}\|$, we obtain: 
\begin{equation*}
\langle \tilde{g}_t, z_t - y_{t+1} \rangle - \frac{1}{\eta}D_R(y_{t+1}||z_t)  \leq \frac{\eta}{2} \|\tilde{g}_t\|_2^2\,.
\end{equation*}
We have proved that: 
\begin{equation*}
\langle \tilde{g}_t, z_t - z_{\delta}^{\star} \rangle 
\leq \frac{D_R(z_{\delta}^{\star}||z_t) - D_R(z_{\delta}^{\star}||y_{t+1})}{\eta} 
+ \frac{\eta}{2}\|\tilde{g}_t\|_2^2\,.
\end{equation*}
Using the previous inequality, we have: 
\begin{equation*}
\langle \tilde{g}_t, z_t - z_{\delta}^{\star} \rangle 
\leq \frac{D_R(z_{\delta}^{\star}||z_{t}) - D_R(z_{\delta}^{\star}||z_{t+1})}{\eta} + \frac{D_R(z_{\delta}^{\star}||z_{t+1}) - D_R(z_{\delta}^{\star}||y_{t+1})}{\eta} + \frac{\eta}{2}\|\tilde{g}_t\|_2^2\,.
\end{equation*}
By $G$-Lipschitzness of $z \mapsto D_R(z_{\delta}^{\star}||z)$, we have: 
\begin{equation*}
D_R(z_{\delta}^{\star}||z_{t+1}) - D_R(z_{\delta}^{\star}||y_{t+1}) 
\leq G \|z_{t+1} - y_{t+1}\|\,,
\end{equation*}
which implies that: 
\begin{equation*}
\langle \tilde{g}_t, z_t - z_{\delta}^{\star} \rangle 
\leq \frac{D_R(z_{\delta}^{\star}||z_{t}) - D_R(z_{\delta}^{\star}||z_{t+1})}{\eta} + \frac{G}{\eta} \|z_{t+1} - y_{t+1}\| + \frac{\eta}{2}\|\tilde{g}_t\|_2^2\,.
\end{equation*}
Summing up the above inequality, the first term telescopes and we obtain:  
\begin{equation*}
\sum_{t=1}^T \langle \tilde{g}_t, z_t - z_{\delta}^{\star} \rangle 
\leq \frac{D_R(z_{\delta}^{\star}||z_1)}{\eta} + \frac{\eta}{2} \sum_{t=1}^T\|\tilde{g}_t\|_2^2 + \frac{G}{\eta} \sum_{t=1}^T \|z_{t+1} - y_{t+1}\|_2\,,
\end{equation*}
which concludes the proof upon recalling the notation $r_{t+1} = z_{t+1} - y_{t+1}\,.$
As for the second bound, it follows immediately from using $\|\tilde{g}_t\| \leq \tilde{G}_F$ and $\|r_{t+1}\| \leq \hat{C}_{\eta}$\,.
\end{proof}

In the next lemma, we show how to control the error sequence $(r_t)$ in Lemma~\ref{lem:perturbed-omd-reg-random-lin-loss} by providing a precise constant $\hat{C}_{\eta}$ which quantifies the approximation error in the bandit setting, as a corollary of Proposition~\ref{prop:dist-minimizers-eta2}. 

\begin{tcolorbox}[colframe=white!, top=2pt,left=2pt,right=2pt,bottom=2pt]
\begin{lemma}[Gradient estimators norm bounds]
\label{lem:grad-estimator-bounds}
Suppose that $\ell_t$ is uniformly bounded by $M$, i.e. $|\ell_t(x)| \leq M, \forall x \in \X$, and the inverse Jacobian of $q$ is bounded by $G$ at any point $x \in \X$ (as in Assumption~\ref{as:reg-R-reparam-q-smoothness}). Then for any $t \geq 1$, we have: 
\begin{align}
\|g_t\| &\leq \hat{G}_F := \frac{dM}{\delta}\,, \quad \|\tilde{g}_t\| \leq \tilde{G}_F := \frac{dM G}{\delta}\,.
\end{align}
\end{lemma}
\end{tcolorbox}

\begin{proof}
Recall from BOGD (Algorithm~\ref{alg4}) that $g_t = \frac{d}{\delta} \ell_t(\hat{x}_t) \zeta_t$. Then, since $\|\zeta_t\| = 1$, we have: 
\begin{equation*}
\|g_t\| = \frac{d}{\delta} |\ell_t(\hat{x}_t)| \leq \frac{dM}{\delta} = \hat{G}_F\,.
\end{equation*}
Then, since $\tilde{g}_t = J_q(x_t)^{-\top} g_t$, we have: 
\begin{equation*}
\|\tilde{g}_t\| \leq G \|g_t\| \leq \frac{dMG}{\delta} = \tilde{G}_F\,.
\end{equation*}
\end{proof}

\begin{tcolorbox}[colframe=white!, top=2pt,left=2pt,right=2pt,bottom=2pt]
\begin{corollary}[Proposition~\ref{prop:dist-minimizers-eta2}
in the bandit setting]
\label{cor:approx-lemma-bandit}
Under Assumptions~\ref{as:reparam}--\ref{as:bounded-euclidean-diam},
let $x_{t+1}$ be the BOGD update, set $z_t=q(x_t)$, and define: 
\[
y_{t+1}
=
\argmin_{y\in\Y_\delta}
\left\{
\langle\tilde g_t,y-z_t\rangle
+
\frac1\eta D_R(y\|z_t)
\right\},
\qquad
\tilde g_t=J_q(x_t)^{-\top}g_t.
\]
Then, with $\hat G_F=dM/\delta$ (as defined in Lemma~\ref{lem:grad-estimator-bounds}), we have for any $t \geq 1$, 
\[
\|y_{t+1}-q(x_{t+1})\|_2
\le
\hat C_\eta
:=
G^5(5\hat G_F^2+\hat G_F^3\eta)\eta^2.
\]
\end{corollary}
\end{tcolorbox}

\begin{proof}
Apply Proposition~\ref{prop:dist-minimizers-eta2} to the restricted map
$q|_{\X_\delta}:\X_\delta\to\Y_\delta$, with common starting point
$z_t=q(x_t)$.
The set $\X_\delta$ is convex, and $\Y_\delta$ is convex by
Assumption~\ref{as:bounded-euclidean-diam}.
The geometric identities and derivative bounds are inherited from
$\X$ and $\Y$, and Lemma~\ref{lem:grad-estimator-bounds} gives
$\|g_t\|_2\le dM/\delta$.
\end{proof}

\noindent\textbf{End of Proof of Theorem~\ref{thm:bandit}.} We are ready to conclude the proof of the theorem using the results established so far. First, it follows from Lemma~\ref{lem:expected-reg-decomp} that: 
\begin{equation}
\label{eq:final-reg-decomp}
\mathbb{E}[R_T] \leq R_T^{(1)} + R_T^{(2)} + R_T^{(3)} + R_T^{(4)}\,.
\end{equation}
We have proved a bound for each one of the error terms in the above bound: 
\begin{enumerate}[label=(\roman*)]
    \item By Lemma~\ref{lem:domain-shrinkage-error}, we have $R_T^{(2)} \leq \delta G_FDT\,.$
    \item By Lemma~\ref{lem:grad-estimation-bias}, we get: $|R_T^{(3)}| \leq \delta G^2 H DT\,.$ 
    \item By Lemma~\ref{lem:smoothing-error}, we obtain: $| R_T^{(4)}| \leq \delta G_FT\,.$
    \item Lemma~\ref{lem:perturbed-omd-reg-random-lin-loss}-\eqref{eq:R1T-bandit} together with Corollary~\ref{cor:approx-lemma-bandit} yield: 
    \begin{equation*}
    R_T^{(1)} \leq \frac{D_1}{\eta} + \frac{\eta \tilde{G}_F^2T}{2} + \frac{\hat{C}_{\eta} G T}{\eta} \leq \frac{D_1}{\eta} + (6 \hat{G}_F^2 + \hat{G}_F^3 \eta) G^6 \eta T\,,
    \end{equation*}
    where the last inequality follows from using $\hat{C}_{\eta} =  G^5 (5\hat{G}_F^2 + \hat{G}_F^3 \eta) \eta^2$ (Corollary~\ref{cor:approx-lemma-bandit}) and upper bounding the resulting constant under the assumptions $G > 1$ (which is without loss of generality), similarly to the deterministic setting. Note here that $\hat{G}_F = \frac{dM}{\delta}$ as shown in Lemma~\ref{lem:grad-estimator-bounds}.   
\end{enumerate}
Combining all the above bounds in \eqref{eq:final-reg-decomp}, we obtain: 
\begin{equation}
\label{eq:regret-upper-bound-proof}
\mathbb{E}[R_T] \leq K_1 \delta T  + \frac{D_1}{\eta} + K_2 \frac{\eta T}{\delta^2} + K_3 \frac{\eta^2 T}{\delta^3}\,, 
\end{equation}
where $K_1 := G_FD + G^2 HD + G_F, K_2 := 6 d^2 G^6 M^2, K_3 = d^3 G^6 M^3\,.$

We choose $(\eta,\delta)$ by balancing the first three terms
of~\eqref{eq:regret-upper-bound-proof}.
The remaining term will be evaluated at the same parameter values
and retained in the final bound. 
Define then for any $(\eta, \delta)$, 
\begin{equation*}
f(\eta, \delta) := K_1 \delta T  + \frac{D_1}{\eta} + K_2 \frac{\eta T}{\delta^2}\,, 
\end{equation*}
which consists of the three first terms in the upperbound \eqref{eq:regret-upper-bound-proof}. We now optimize this function w.r.t $(\eta, \delta)$. By the first order optimality conditions, the optimal pair $(\eta_{\star}, \delta_{\star})$ minimizing $f(\eta, \delta)$ satisfies: 
\begin{equation*}
\frac{\partial f}{\partial \eta} (\eta_{\star}, \delta_{\star}) = - \frac{D_1}{\eta_{\star}^2} + \frac{K_2 T}{\delta_{\star}^2} = 0\,,\quad
\frac{\partial f}{\partial \delta} (\eta_{\star}, \delta_{\star}) 
= K_1 T - \frac{2 K_2 \eta_{\star} T}{\delta_{\star}^3} = 0\,.
\end{equation*}
Rewriting these conditions yields: 
\begin{equation*}
\eta_{\star} = \sqrt{\frac{D_1 \delta_{\star}^2}{K_2 T}}\,, \quad \eta_{\star} = \frac{K_1 \delta_{\star}^3}{2 K_2}\,.
\end{equation*}
Solving for $\delta_{\star}$ as a function of $T$ (eliminating $\eta$) and then plugging back $\delta_{\star}$ in the first expression of~$\eta_{\star}$ yields: 
\begin{equation*}
\delta_{\star} = \left(\frac{4 K_2 D_1}{K_1^2 T}\right)^{1/4}\,, \quad \eta_{\star} = \left(\frac{4 D_1^3}{K_1^2 K_2 T^3}\right)^{1/4}\,.
\end{equation*}
Now, plugging back these values of $\eta_{\star}$ and $\delta_{\star}$ in \eqref{eq:regret-upper-bound-proof}, we obtain: 
\begin{equation*}
\mathbb{E}[R_T] \leq 2 \sqrt{2} (K_1^2 K_2D_1)^{1/4} T^{3/4} + \frac{1}{\sqrt{2}}\left(\frac{K_1^{1/2} K_3 D_1^{3/4}}{K_2^{5/4}}\right) T^{1/4}\,.
\end{equation*}
which concludes the proof of Theorem~\ref{thm:bandit} since $T^{1/4} \leq T^{3/4}$ for $T \geq 1$.



\end{document}